%% file: main.tex
\documentclass{article}

\usepackage[final,nonatbib]{neurips_2022}
\usepackage[numbers]{natbib}

\usepackage[utf8]{inputenc} 
\usepackage[T1]{fontenc}    
\usepackage{hyperref}       
\usepackage{url}            
\usepackage{booktabs}       
\usepackage{amsfonts}       
\usepackage{nicefrac}       
\usepackage{microtype}      
\usepackage[dvipsnames,x11names]{xcolor}         
\usepackage[pdftex]{graphicx}
\usepackage{mathtools}
\usepackage{amsmath}
\usepackage{amssymb}
\usepackage{hyperref}
\usepackage{amsthm}
\usepackage[normalem]{ulem}
\usepackage[capitalize]{cleveref} 
\usepackage[frozencache,cachedir=.]{minted} 
\usepackage{textgreek} \usepackage[LGR, T1]{fontenc}
\usepackage{caption}
\usepackage{subcaption}
\usepackage{varwidth}
\usepackage{verbatim}
\usepackage{bm}
\usepackage{bbold}

\usepackage{todonotes}
\usepackage{transparent}
\usepackage{accents}

\usepackage[resetlabels,labeled]{multibib}
\newcites{SI}{Supplemental References}

\title{Automatic Differentiation of Programs with\\ Discrete Randomness}

\author{Gaurav~Arya \\
  Massachusetts Institute of Technology, USA \\ 
  \texttt{aryag@mit.edu}\\
  \And
  Moritz~Schauer \\
  Chalmers University of Technology, Sweden \\
  University of Gothenburg, Sweden \\ 
  \texttt{smoritz@chalmers.se}\\
  \And
  Frank~Sch\"afer \\
  Massachusetts Institute of Technology, USA \\ 
  University of Basel, Switzerland\\
  \texttt{franksch@mit.edu}\\
  \And
  Chris~Rackauckas \\
  Massachusetts Institute of Technology, USA\\
  Julia Computing Inc., USA\\
  Pumas-AI Inc., USA\\
  \texttt{crackauc@mit.edu}\\
}

\include{preamble}

\begin{document}

\maketitle

\vspace{-15px}
\begin{abstract}
Automatic differentiation (AD), a technique for constructing new programs which compute the derivative of an original program, has become ubiquitous throughout scientific computing and deep learning due to the improved performance afforded by gradient-based optimization. However, AD systems have been restricted to the subset of programs that have a continuous dependence on parameters. Programs that have discrete stochastic behaviors governed by distribution parameters, such as flipping a coin with probability $p$ of being heads, pose a challenge to these systems because the connection between the result (heads vs tails) and the parameters ($p$) is fundamentally discrete. In this paper we develop a new reparameterization-based methodology that allows for generating programs whose expectation is the derivative of the expectation of the original program. We showcase how this method gives an unbiased and low-variance estimator which is as automated as traditional AD mechanisms. We demonstrate unbiased forward-mode AD of discrete-time Markov chains, agent-based models such as Conway's Game of Life, and unbiased reverse-mode AD of a particle filter. Our code package is available at \url{https://github.com/gaurav-arya/StochasticAD.jl}.
\end{abstract}

\section{Introduction}

\input{intro.tex}

\input{related_work.tex}

\vspace{-5px}
\section{Composable derivatives of stochastic programs}
\vspace{-5px}
\label{sec:motivation}

\input{motivation.tex}

\section{Automatic differentiation of stochastic programs}

\label{sec:demos}

\input{demos.tex}
\vspace{-4px}
\section{Limitations and outlook}
\vspace{-4px}

\label{sec:outlook}

\input{outlook}
\vspace{-4px}
\begin{ack}
\vspace{-4px}
We thank Alan Edelman, Guillaume Dalle, and the anonymous reviewers for their feedback, Emile van Krieken for helpful discussions regarding composing gradient estimators, and Simeon Schaub for help with the package. We acknowledge the MIT SuperCloud and Lincoln Laboratory Supercomputing Center for providing HPC resources that have contributed to the research results reported within this paper.
This material is based upon work supported by the National Science Foundation OAC-1835443, SII-2029670, ECCS-2029670, OAC-2103804, and PHY-2021825; the Advanced Research Projects Agency-Energy DE-AR0001211 and DE-AR0001222; the Defense Advanced Research Projects Agency (DARPA) HR00112290091; the United States Artificial Intelligence Accelerator FA8750-19-2-1000; the Chalmers AI Research Centre; and the Swiss National Science Foundation 51NF40-185902. The views and opinions of authors expressed herein do not necessarily state or reflect those of the United States Government or any agency thereof.
\end{ack}

\bibliographystyle{unsrt}
\bibliography{refs1}

\input{checklist}

\newpage

\appendix

In the supplement, we:
\begin{itemize}
\item provide a number of examples of stochastic derivatives and present hand-worked examples of how they compose (\cref{sec:derived}),
\item prove our formal statements regarding stochastic derivatives (\cref{sec:proofs}), 
\item provide an introduction to the particle filter methodology, and show how smoothed stochastic derivatives leads to unbiased differentiation of the resampling step (\cref{sec:SI_filter}),
\item provide details about the hardware used for the experiments of the main text and the software dependencies of \texttt{StochasticAD.jl} (\cref{sec:implementation}).
\end{itemize}

\input{derived.tex}
\input{proofs.tex}

\input{filter.tex}

\input{backmatter.tex}

\bibliographystyleSI{unsrt}
\bibliographySI{refs2}

\end{document}

%% file: intro.tex
Automatic differentiation (AD) is a technique for taking a mathematical program $X(p)$ and generating a new program $\tilde{X}(p) = \frac{\dd X}{\dd p}$ for computing the derivative \cite{griewank1989automatic,baydin2018automatic}. AD is widely used throughout machine learning and scientific computing due to the increased performance of gradient-based optimization techniques compared to derivative-free methods \cite{corliss2002automatic}. However, if $X(p)$ returns the flip of a coin with probability $p$ of receiving a $1$ and probability $1-p$ of receiving a $0$, it is clear that $\frac{\dd X}{\dd p}$ is not defined in the classical sense. But when attempting to calibrate the parameter $p$ to data, one may wish to fit the model using statistical quantities, e.g. find $p$ such that the average of $X(p)$ is close to the average sum of $N$ real-world coin flips. Given this use case, can one automatically construct a program $\tilde{X}(p)$ that computes the derivative of the statistical quantities, i.e. $\EE[\tilde{X}(p)] = \frac{\dd\EE[X(p)]}{\dd p}$?

\begin{figure}
\centering
\includegraphics[width=\textwidth]{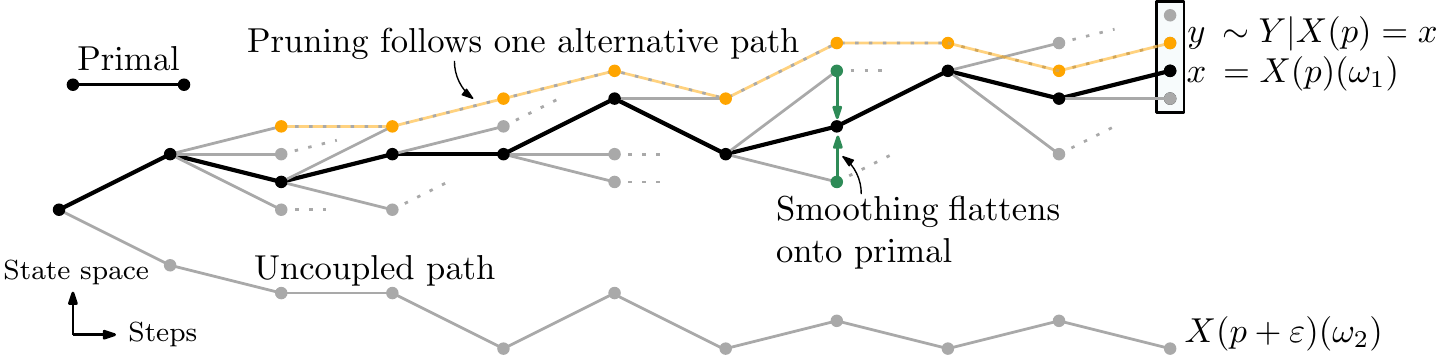}
\caption{Qualitative sketch of our method and comparison to finite differences.
The primal computation (solid black line) samples $X(p)$ with random number sequence $\omega_1$. Black-box finite differences samples the perturbed program $X(p+\epsilon)$ with an \emph{independent} random number sequence $\omega_2$ (bottom) for some \emph{finite} choice of $\epsilon$. In contrast, the component $Y$ of the stochastic derivative of the program (\cref{sec:motivation}) considers the effect of the minimal possible perturbations (gray lines diverging from primal path) to the original program which could stem from a truly \emph{infinitesimal} change in the input $p$.}
\label{fig:path}
\end{figure}

A na\"ive solution to this problem would be to use finite differences, i.e.:
\begin{equation}
    \frac{\dd\EE[X(p)]}{\dd p} \approx \frac{\EE[X(p + \epsilon)] - \EE[X(p)]}{\epsilon}.
\end{equation}
Finite differences' major issue in the context of stochastic programs is that this calculation does not correlate the calculation of $\EE[X(p + \epsilon)]$ with the calculation of $\EE[X(p)]$. This leads to a large variance in the finite-difference estimator~\cite{asmussen2007stochastic} that goes to $\infty$ as $\epsilon \rightarrow 0$. This issue of unbounded variance could be solved if one could run the perturbed program with the same set of random numbers and directly estimate $\EE[X(p + \epsilon) - X(p)]$ by a method that is well-posed in the limit of $\varepsilon \to 0$ (\cref{fig:path}).

We demonstrate how to automatically construct a new stochastic program $\tilde X(p)$ whose expectation satisfies $\EE[\tilde{X}(p)] = \frac{\mathrm{d}\EE[X(p)]}{\mathrm{d}p}$. We derive this through a technique which we term the \emph{stochastic derivative}, propagating the proportional probability of differing event outcomes due to infinitesimal changes in $p$. Our technique has the following design goals:

\begin{itemize}
    \item \textbf{Composability.} We differentiate stochastic programs $X(p)$ which are themselves composed of many ``elementary'' stochastic programs, including samples from discrete distributions such as the Bernoulli, Poisson, and Geometric, and samples from continuous distributions. These may be composed (chained, added, etc.) arbitrarily, with computational cost independent of how they are composed.\item \textbf{Unbiasedness.} The program's discrete structure is preserved, leading to a provably unbiased estimator that avoids continuous relaxations and tunable or learned internal parameters.
    \item \textbf{Low variance.} We consider correlated paths through the computation that are linked by the smallest possible perturbation (\cref{fig:path}), generalizing the widely-used pathwise gradient estimator for continuous randomness~\cite{mohamed2020monte} to the discrete case.
\end{itemize}

We show the utility of our technique by demonstrating forward-mode AD of stochastic simulations like inhomogeneous random walks and agent-based models such as the Game of Life. To achieve~$\mathcal{O}(1)$ computational overhead for these applications, we combine stochastic derivatives with an online strategy we call \emph{pruning}. We also demonstrate a straightforward way to perform reverse-mode AD from our approach via \emph{smoothing}, allowing us to derive from first principles the biased but empirically successful straight-through gradient estimator \cite{bengio} as a special case and to construct an unbiased end-to-end reverse-differentiable particle filter, recovering a technique discovered in \cite{scibior2021differentiable}.
\ We provide an open-source implementation of the method, \texttt{StochasticAD.jl}, for readers to explore the technique on their own applications.

%% file: related_work.tex
\vspace{-5px}
\subsection{Related work}
\vspace{-5px}

Gradient estimators for stochastic functions can be divided into three different classes~\cite{mohamed2020monte}:  score-function~\cite{glynn1990likelihood,williams1992simple,kleijnen1996optimization}, measure-valued~\cite{pflug1989sampling,heidergott2008measure}, and pathwise~\cite{glasserman1991gradient,kingma2013auto} gradient estimators. 
When it comes to discrete randomness, the score-function method is a popular general-purpose choice because it is unbiased and can be composed through stochastic computation graphs~\cite{schulman2015gradient}. 
However, the score-function method does not search for correlated paths through the computation and thus suffers from high variance, making gradient computation with discrete variables challenging. A number of techniques (e.g. REBAR~\cite{tucker2017rebar}, RELAX~\cite{grathwohl2017backpropagation}) have been introduced that rely on control variates for variance reduction~\cite{glynn2002some}. 
Measure-valued derivatives also have no notion of intrinsic coupling, though coupling can be achieved using common random numbers for certain distributions~\cite{heidergott2008sensitivity}.
As an alternative direction, Gumbel-Softmax~\cite{jang2016categorical} considers a continuous relaxation of discrete programs so that a pathwise gradient estimator, based on the ``reparameterization trick''~\cite{kingma2013auto}, can be applied. However, such methods face a bias-variance tradeoff and are inapplicable to discrete programs that cannot be continuously relaxed. 
Our stochastic derivatives also extend the pathwise gradient estimator to discrete programs but do so unbiasedly.
 The conceptual starting point of our approach is finite differences with common random numbers~\cite{anderson2012efficient,thanh2018efficient,glasserman1992some}, whose ideas have also been extended by direct optimization~\cite{NIPS2010_ca8155f4,NEURIPS2020_d1e7b08b}, but
 crucially we show how to take the exact limit of step size $\varepsilon \to 0$ even in the discrete case.
\ The field of smoothed perturbation analysis~\cite{gong1987smoothed,glasserman1990smoothed} develops a mathematically equivalent object to our smoothed stochastic derivative based on conditional expectations, which is a special case of our formalism, and has also considered a randomized approach similar to our pruning technique in the context of generalized semi-Markov processes~\cite{Fu1997}.  
However, these ideas have not previously been applied to construct a general-purpose AD method via rigorous composition rules and the algorithms and data structures to realize the approach automatically. 
\looseness=-1

Although our method can be used to hand-derive a gradient estimator, the main feature is composability through user-written functions, enabling an automated mechanism. While mainstream AD frameworks do not support unbiased differentiation of discrete random programs, Storchastic~\cite{krieken2021storchastic} is a specialized framework for AD of stochastic computation graphs~\cite{schulman2015gradient} where the user can specify which estimator to use at each node, as well as any tunable hyperparameters. Storchastic implements an exhaustive set of prior gradient estimation
methods at each sampling step. However, the runtime of the derivative estimate is in general exponential in the length of the largest chain of stochastic nodes, an artifact of the way many prior gradient estimators compose~\cite{krieken2021storchastic,StorchasticReview}. In Section \ref{sec:demos} we demonstrate stochastic AD that matches the computational complexity of deterministic AD even when discrete random functions are chained together, alleviating these performance issues.

%% file: motivation.tex
  In this section, we develop the notion of a \emph{stochastic derivative} for programs containing discrete randomness. We shall motivate this object as a natural generalization of the pathwise gradient estimator for continuous random programs, and present the key ideas underpinning the formalism.\ When describing infinitesimal asymptotics, we say a function $g(\varepsilon)$ is $\mathcal{O}(\varepsilon)$ if $|g(\varepsilon)| \leq C|\varepsilon|$ for some real $C$ and all sufficiently small $\varepsilon$.\ Colloquially, we describe quantities that are $\mathcal{O}(\varepsilon)$ as ``infinitesimal''.

\vspace{-4.5px}
\subsection{Infinitesimally perturbing a stochastic program}
\vspace{-4.5px}
\label{sec:infinitesimals}

We are interested in differentiating stochastic programs, formally defined below. 
A stochastic program can be thought of as a map from an input $p$ to a random variable $X(p)$.
Here, $X(p)$ can represent either an ``elementary'' program such as a draw from a Bernoulli distribution, or the full user-provided stochastic program represented using many elementary programs, e.g. a simulation of a random walk. 
\begin{definition}
\label{def:program}
A \emph{stochastic program} $X(p)$ is a stochastic process with values in a Euclidean space $E$, whose index set $I$ is either an open subset of a Euclidean space or a closed real interval.\end{definition}
Let $\Omega$ be the sample space, equipped with a probability distribution $\mathbb{P}$.
To sample $X(p)$ at input $p \in I$, which we call the \emph{primal evaluation}, one should imagine a sample $\omega$ being randomly chosen from $\Omega$ according to $\mathbb{P}$ to produce an output $X(p)(\omega) \in E$. Note that $\mathbb{P}$ is independent of $p$ and $X(p)$ is a map $\Omega \to E$; such a formulation has been called the ``reparameterization trick''~\cite{kingma2013auto}. For example, a Bernoulli distribution $\operatorname{Ber}(p)$ can be represented by choosing a uniform random $\omega \in [0,1]$ and defining $X(p)(\omega) = \mathbf{1}_{[1-p, 1]}(\omega)$, where $\mathbf{1}_{S}$ is the indicator function for a set $S$ (see \cref{fig:ber1}).

\begin{figure}
     \centering
     \begin{subfigure}[b]{0.32\textwidth}
         \centering
         \includegraphics[height=4.5cm]{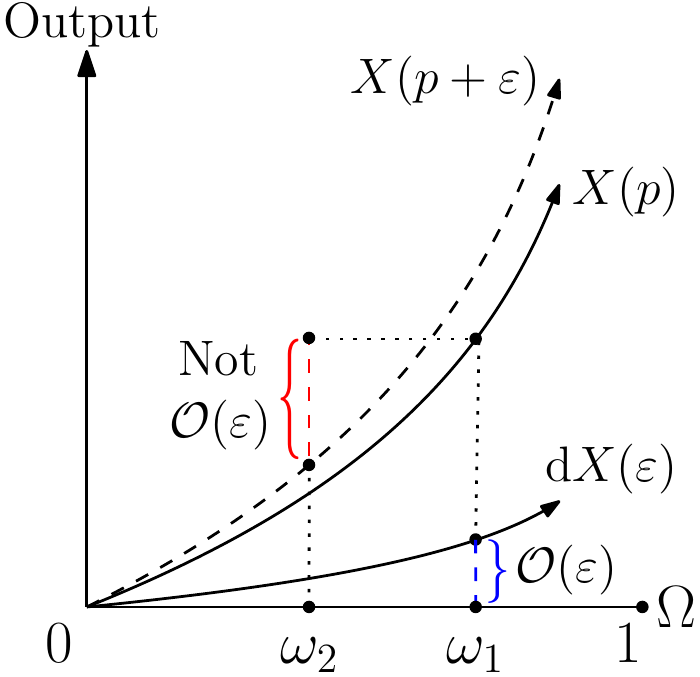}
         \caption{$X(p) \sim \operatorname{Exp}(p)$.}
         \label{fig:exp}
     \end{subfigure}
     \hfill
     \begin{subfigure}[b]{0.32\textwidth}
         \centering
         \includegraphics[height=4.5cm]{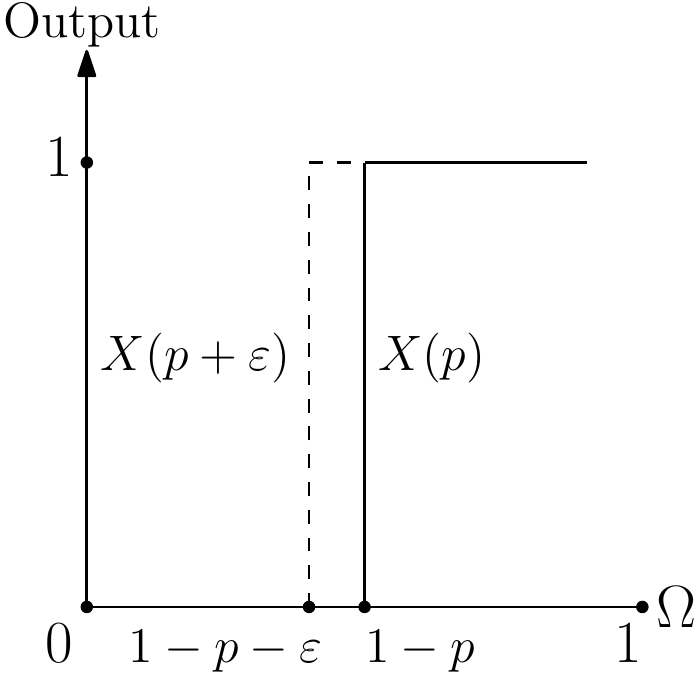}
         \caption{$X(p) \sim \operatorname{Ber}(p)$.}
         \label{fig:ber1}
     \end{subfigure}
     \begin{subfigure}[b]{0.32\textwidth}
         \centering
         \includegraphics[height=4.5cm]{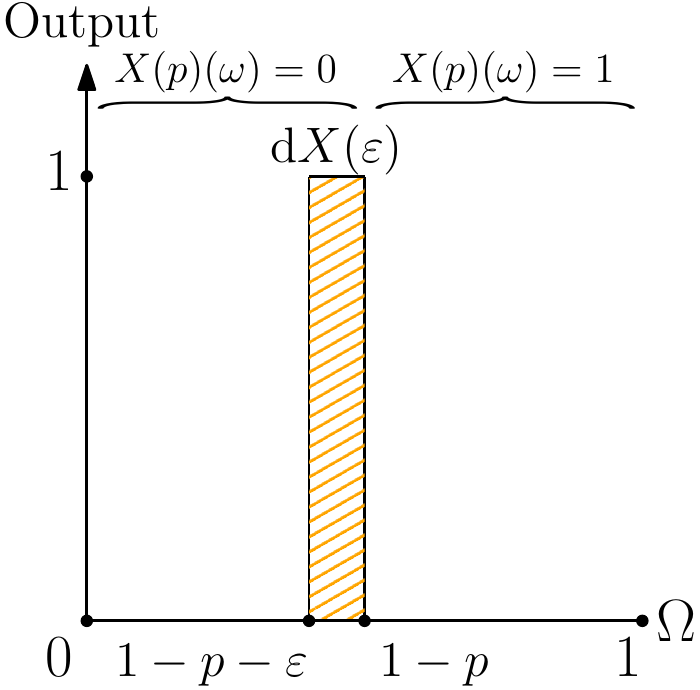}
         \caption{$X(p) \sim \operatorname{Ber}(p)$.}
         \label{fig:ber2}
     \end{subfigure}
     \hfill
        \caption{Illustrations of $\dd X(\varepsilon)$ for a small finite $\varepsilon$, assuming $\varepsilon > 0$ for simplicity. 
\textbf{(a)} The $\mathcal{O}(\varepsilon)$ range of $\dd X(\varepsilon)$ for continuous $X(p)$. The blue bracket indicates a sample from $\dd X(\varepsilon)$, while the red bracket indicates a sample obtained by black-box finite differences.
\textbf{(b)} The original program $X(p)$ and the perturbed program $X(p+\varepsilon)$ for discrete (Bernoulli) $X(p)$.
        \textbf{(c)} Only those samples $\omega$ below the hatched area cause a non-zero change in output. The hatched area equals $\EE[\dd X(\varepsilon)]$, the expected size of a change for randomly drawn $\omega$.
} 
        \label{fig:infinitesimal}
\end{figure}

At fixed $p \in I$ we define the differential $\dd X(\varepsilon)$, which is itself a stochastic program:
\begin{equation}
    \dd X(\varepsilon) = X(p+\varepsilon) - X(p). \label{def:infinitesimal}
\end{equation}
Let us now restrict our attention to the case where $I$ is a closed real interval, so that $p, \varepsilon \in \mathbb{R}$. The sensitivity of stochastic programs $Z$ with more general index sets can be understood at an input $\mathbf{u}$ by studying at $p = 0$ the directional perturbation $X(p) = Z(\mathbf{u} + p \mathbf{v})$ in a direction $\mathbf{v}$.

For the full user-provided stochastic program, we are interested in computing $\frac{\dd\EE[X(p)]}{\dd p}$, which is related to $\EE[\dd X(\varepsilon)]$ by
\begin{equation}
\frac{\dd \EE[X(p)]}{\dd p} = \lim_{\varepsilon \to 0} \frac{\EE[\dd X(\varepsilon)]}{\varepsilon}.
\label{eq:EdX}
\end{equation}
We thus expect $\dd X(\varepsilon)$ to have ``infinitesimal'' $\mathcal{O}(\varepsilon)$ expectation as $\varepsilon$ approaches 0. Let us consider the form of $\dd X(\varepsilon)$ for two elementary programs: a program returning a sample from the continuous Exponential distribution $\operatorname{Exp}(p)$ with scale $p$ and a program returning a sample from the discrete Bernoulli distribution $\operatorname{Ber}(p)$ with success probability $p$. In both cases, the program is parameterized via the inversion method~\cite{devroye2006nonuniform} over the sample space $\Omega = [0,1]$, which chooses $\mathbb{P}$ to be uniform over $[0,1]$ and $X(p)(\omega)$ to be non-decreasing in $\omega$. 
For $X(p) \sim \operatorname{Exp}(p)$, $X(p)(\omega) = -p\log(1 - \omega)$ is differentiable in $p$. Thus, the range of $\dd X(\varepsilon)$ is $\mathcal{O}(\varepsilon)$ for any fixed $\omega$, as illustrated in \cref{fig:exp}. 

In contrast, consider $X(p) \sim \operatorname{Ber}(p)$, with $X(p)(\omega) = \mathbf{1}_{[1-p, 1]}(\omega)$. As shown in \cref{fig:ber1}, as $\varepsilon \to 0$ the output of $X(p)$ at a random $\omega \in \Omega$ almost surely does not change in $X(p+\varepsilon)$. Specifically, $\dd X(\varepsilon)$ is 0 with probability $1 - \varepsilon \approx 1$, and assumes the value $1$ with probability $\varepsilon$ (\cref{fig:ber2}). This is the fundamental challenge presented by discrete randomness: a $\mathcal{O}(\varepsilon)$ change in the input turns into a \emph{finite} perturbation to the output, where ``finite'' means that it is non-vanishing as $\varepsilon \to 0$.\ The perturbation occurs with ``infinitesimal'' $\mathcal{O}(\varepsilon)$ probability, contributing to the $\mathcal{O}(\varepsilon)$ expectation of $\dd X(\varepsilon)$.
\looseness=-1

\vspace{-3px}
\subsection{Coupling to the primal}
\vspace{-3px}
\label{sec:coupling}

To produce low-variance estimates of the derivative of $X(p)$, it is important that the primal and derivative computations are coupled by a shared source of randomness. To see this, let us understand what happens when they are entirely uncoupled. A black-box finite difference approach with step size $\varepsilon$ would independently sample $\omega_1$ and $\omega_2$ from $\Omega$ (\cref{fig:exp}, red), computing the derivative estimate $\left(X(p+\varepsilon)(\omega_2) - X(p)(\omega_1)\right)/\varepsilon$. However, since the samples are independent, the variance of the estimator is of order $1/\varepsilon^2$, so we are forced to pick a finite $\varepsilon$ to balance a bias-variance tradeoff. 

This motivates using the same random sample for the primal and derivative computations (\cref{fig:exp}, blue). 
For continuous randomness, taking the limit of this approach as $\varepsilon \to 0$ leads to the widely-used pathwise gradient estimator $\delta$, given as the almost sure limit of $\dd X(\varepsilon) / \varepsilon$ as $\varepsilon \to 0$ (i.e.~the pointwise derivative of $X(p)$ at each fixed $\omega$), so that
\begin{equation}
   \frac{\dd \EE[X(p)]}{\dd p} = \lim_{\varepsilon \to 0} \frac{\EE[\dd X(\varepsilon)]}{\varepsilon} \stackrel{?}{=} \EE\left[ \lim_{\varepsilon \to 0} \frac{\dd X(\varepsilon)}{\varepsilon}\right] = \EE\left[\delta\right].
   \label{eq:exchange}
\end{equation}
But, considering a simple Bernoulli variable $X(p) \sim \operatorname{Ber}(p)$ as in \cref{fig:ber2}, we see how this approach is ill-suited for the discrete case! As $\varepsilon$ approaches 0, the differential $\dd X(\varepsilon)$ is non-zero with infinitesimal~$\mathcal{O}(\varepsilon)$ probability.\ This means that $\delta$ is almost surely 0, while the true derivative of $\EE[X(p)] = p$ is $1$. The finite perturbation is neglected. 

Thus, the pathwise gradient estimator needs to be modified to handle discrete randomness. The issue with the interchange of limit and expectation in (2.3) is that $\dd X(\varepsilon) / \varepsilon$ is unbounded in $\varepsilon$ in the presence of a finite perturbation. This motivates explicitly considering the event of a large jump in~$\dd X(\varepsilon)$, as characterized by the event $A_B(\varepsilon) = 
\{|\dd X(\varepsilon)| > B|\varepsilon|\}$ for a chosen random bound $B > |\delta|$. The event $A_B(\varepsilon)$ has $\mathcal{O}(\varepsilon)$ probability, which is vanishingly small, but its contribution to the derivative estimate cannot be neglected because it contains finite perturbations. This motivates sampling from this part of the probability space \emph{separately}. Formally, we introduce a random weight $w \in \mathbb{R}$ and alternate value $Y \in E$ that characterize the sensitivity of $X(p)$ when the probability space is \emph{restricted} to $A_B(\varepsilon)$ [as given by the r.h.s. of \eqref{eq:deriv} below], forming the ``stochastic derivative'':


\begin{definition}[Stochastic derivative]
\label{def:deriv}
Suppose $X(p) \in E$ is a stochastic program with index set $I$ a closed real interval. 
We say that the triple of random variables $(\delta, w, Y)$, with $w \in \mathbb{R}$ and $Y \in E$, is a right (left) \emph{stochastic derivative} of $X$ at the input $p \in I$ if $\dd X(\varepsilon) / \varepsilon \to \delta$ almost surely as $\varepsilon \to 0$, and there is an integrable (i.e.~of bounded expectation) random variable $B > |\delta|$ such that for all bounded functions $f\colon E \to \mathbb{R}$ with bounded derivative it holds almost surely that
\begin{equation}
\label{eq:deriv}
\EE\left[w \left(f(Y) - f(X(p))\right) \xmiddle| X(p) \right] = \lim_{\varepsilon \to 0^{+/-}} \EE\left[\frac{f(X(p + \varepsilon)) - f(X(p))}{\varepsilon} \ind_{A_B(\varepsilon)} \xmiddle| X(p)\right],
\end{equation}
with limit taken from above (below), where $\PP\left(A_B(\varepsilon) \xmiddle| X(p)\right)/\varepsilon$ is dominated by an integrable random variable for all $\varepsilon > 0$ ($\varepsilon < 0$). \end{definition}
A stochastic derivative may be collapsed into an unbiased estimator of the derivative of $\EE[X(p)]$.
\begin{proposition}[Unbiasedness]
\label{prop:unbiased}
If $(\delta, w, Y)$ is a stochastic derivative of $X(p)$ at $p$, it holds that
\begin{equation}
\label{eq:unbiased}
\frac{\dd\EE \left[X(p)\right]}{\dd p} = \EE[\delta + w\left(Y-X(p)\right)].
\end{equation}
\end{proposition}
\begin{proof}[Proof sketch]
With $f$ as identity, by \cref{def:deriv} the sensitivity of $X(p)$ over $A_B(\varepsilon)$ is given by $w(Y-X(p))$. Given the complement event $A_B^c(\varepsilon)$, it holds that $|\dd X(\varepsilon) / \varepsilon| \leq B$: a dominated convergence argument shows that the sensitivity of $X(p)$ is then given by its almost-sure derivative~$\delta$.
\end{proof}
\vspace{-10px}
\cref{thm:main} shows the existence of the stochastic derivative for a stochastic program, subject to technical assumptions given in our formal treatment (\cref{sec:proofs}). \begin{theorem}[Existence, simplified]
\label{thm:main}
Given a sufficiently regular stochastic program $X(p)$ with index set $I$ a closed interval $[a,b] \subset \mathbb{R}$, there exists a right stochastic derivative $(\delta, w_R, Y_R)$ with $w_R \geq 0$ at any $p \in [a,b)$ and a left stochastic derivative $(\delta, w_L, Y_L)$ with $w_L \le 0$ at any $p \in (a,b]$.
\end{theorem}
\begin{proof}[Proof sketch.]
The proof is by construction: at a high level, $w$ is the derivative of the probability of a large jump, while $Y$ follows the distribution of the possible jumps, conditional on a jump happening.  Specifically, $Y$ has distribution given as the limit as $\varepsilon \to 0$ of the conditional distribution of $X(p + \varepsilon) = X(p) + \dd X(\varepsilon)$ given the event $A_B(\varepsilon)$ and the outcome of~$X(p)$. The weight $w$ is given as the derivative w.r.t. $\varepsilon$ of the probability $\PP\left(A_B(\varepsilon) \xmiddle| X(p)\right)$ that $\dd X(\varepsilon)$ jumps by a non-infinitesimal amount, conditional on $X(p)$. Essentially, since $\PP\left(A_B(\varepsilon) \xmiddle| X(p)\right) \approx w\varepsilon$, multiplying by $w$ bridges the gap between \emph{conditioning} on $A_B(\varepsilon)$ [recall $Y$ is constructed conditional on $A_B(\varepsilon)$] and simply restricting the probability space to $A_B(\varepsilon)$ [i.e. multiplying by $\ind_{A_B(\varepsilon)}$ as in \eqref{eq:deriv}].
\end{proof}

\begin{example}[Right stochastic derivative of Bernoulli variable]
\label{ex:ber}
Suppose $X(p) \sim \operatorname{Ber}(p)$, parameterized via the inversion method, and take $\epsilon > 0$. As shown in \cref{fig:ber2}, $\dd X(\varepsilon)$ is parameterized as,
\begin{equation}
\dd X(\varepsilon)(\omega) = 
\begin{cases}
1 & \text{ if $1-p-\varepsilon \leq \omega < 1-p$,}\\
0 & \text{ otherwise.}
\end{cases}
\end{equation}
Given the event $X(p) = 1$, we have that $\omega \ge 1-p$ and $\dd X(\varepsilon)$ is deterministically 0. On the other hand, given $X(p) = 0$, we have $\omega < 1-p$, so with probability $\varepsilon/(1-p)$ the differential $\dd X(\varepsilon)$ assumes a value of 1, i.e. the Bernoulli variable flips from 0 to 1. Thus, we may construct a right stochastic derivative $(0, w_R, Y_R)$ of $X(p)$ by letting $w_R = 1/(1-p), Y_R = 1$ conditionally on $X(p) = 0$ and $w_R = Y_R = 0$ conditionally on $X(p) = 1$. For a concrete example, with $p = 0.6$, we have $w_R = 2.5, Y_R=1$ conditionally on $X(p) = 0$; \texttt{StochasticAD.jl} prints this as \texttt{"0 + (1 with probability 2.5ε)"} (as explained further in \cref{sec:toy}).

\end{example}

We give a number of examples of stochastic derivatives in \cref{sec:derived}. It is of crucial importance that $w$ and $Y$ depend conditionally on the output of $X(p)$. 
In particular, given a primal evaluation~${X(p) = x}$, the forms of $w$ and $Y$ depend only on the distribution of the perturbed program on the set $\{\omega: X(p)(\omega) = x\} \subset \Omega$, which elegantly generalizes the coupling achieved by the pathwise gradient estimator $\delta$ at each fixed $\omega \in \Omega$.
Intuitively, this coupling allows us to consider only the smallest possible perturbations to the program in the derivative computation (\cref{fig:path}), and thereby achieve variance reduction without resorting to continuous relaxations. For example, for a binomial variable $X(p) \sim \operatorname{Bin}(n,p)$ parameterized via the inversion method (a natural parameterization for maximizing coupling), it holds that $Y \in \{X(p)-1, X(p)+1\}$. We show in \cref{ex:int} that our gradient estimator for such a binomial variable has variance of order $n$, whereas the score function estimator has variance of order $n^3$, justifying this intuition.

\subsection{Composition of stochastic derivatives}
\label{sec:composition}

\cref{prop:unbiased} shows that the derivative estimate produced by stochastic derivatives is correct in expectation. But this is insufficient to ensure \emph{composition}, i.e. a stochastic derivative ``chain rule''. While \cref{def:deriv} requires composition through deterministic test functions $f$ to enforce a sufficiently strict definition, \cref{prop:composition} provides a general-purpose composition result through any program which has a stochastic derivative, as well as multidimensional programs with directional stochastic derivatives (i.e. stochastic derivatives of directional perturbations of the program).
\begin{theorem}[Chain rule, simplified]
\label{prop:composition}
Consider independent stochastic programs $X_1$ and $X_2$ and their composition $X_2 \circ X_1$. Suppose that $X_1$ has a right (left) stochastic derivative at $p \in \mathbb{R}$ given by $(\delta_1, w_1, Y_1)$, and $X_2$ has a right stochastic derivative $(\delta_2, w_2, Y_2)$ in the direction $\widehat{\delta}_1 = \delta_1 / |\delta_1|$ given conditionally on its input $X_1(p)$.
\ Then, under regularity and integrability assumptions, the stacked program $[X_1; X_2 \circ X_1]$ has a right (left) stochastic derivative $(\delta, w, Y)$  at $p$ where $\delta = \left[\delta_1 ; |\delta_1|\delta_2\right]$, \begin{equation} Y \quad = \quad \begin{cases}
    \label{eq:cases}
    [Y_1; X_2(Y_1)] & \text{ with probability }\quad\dfrac{w_1}{w_1 + |\delta_1| w_2}, \\
    [X_1(p); Y_2] & \text{ with probability }\quad \dfrac{|\delta_1| w_2}{w_1 + |\delta_1| w_2},
    \end{cases}
\end{equation}
and $w = w_1 + |\delta_1| w_2$. 
\end{theorem}
\begin{proof}[Proof sketch]
Let $A_1(\varepsilon)$ be the event of a jump in $X_1$ and $A_2(\varepsilon)$ be the event of a jump in $X_2$ when its input is perturbed by $\varepsilon \delta_1$. Given $A_1^c(\varepsilon) \cap A_2^c(\varepsilon)$ (a jump in neither), a dominated convergence argument implies that $X_2 \circ X_1$ does not jump either. Given $A_1(\varepsilon)$, \cref{def:deriv} yields that the sensitivity of $X_2 \circ X_1$ is described by alternate values $Y_1$ with weight $w_1$, while given $A_2(\varepsilon)$ the sensitivity of $X_2 \circ X_1$ is described by alternate values $Y_2$ with weight $|\delta_1|w_2$. We may then form $w$ as the sum of these weights and $Y$ as a weighted distribution over the two cases. Crucially, we may neglect the sensitivity of $X_2 \circ X_1$ given $A_1(\varepsilon) \cap A_2(\varepsilon)$ (a jump in both), as this event has probability~$\mathcal{O}(\varepsilon^2)$. This prevents a combinatorial explosion in the complexity of~$Y$.
\end{proof}

\subsection{Smoothed stochastic derivatives}
\label{sec:smooth}

The reparameterization trick neglects finite perturbations, while stochastic derivatives precisely capture all possible finite perturbations. There exists a middle ground between these methods: one may take a conditional expectation on $X(p)$  
so that finite perturbations have been ``smoothed'' into infinitesimal ones. (E.g. \texttt{"0 + (1 with probability 2.5ε)"} becomes \texttt{"0 + 2.5ε"}.)
\begin{definition}[Smoothed stochastic derivative]
\label{def:smooth}
For a stochastic program $X(p)$ with a right (left) stochastic derivative $(\delta, w, Y)$ at input $p$, a right (left) smoothed stochastic derivative $\tilde{\delta}$ of $X$ at input $p$ is given as
\begin{equation}
\label{eq:defsmooth}
    \tilde \delta = \EE\left[\delta + w (Y- X(p)) \xmiddle| X(p)\right].\end{equation}
\end{definition}
Smoothed stochastic derivatives easily permit reverse-mode AD instead of forward-mode, as they have the same form as the usual derivative. However, they enjoy more limited composition properties: they propagate exactly through differentiable functions $f$ that are linear over the conditional distribution of $Y$ given~$X(p)$ via the standard chain rule, as we prove in \cref{app:smoothing}. Due to the coupling of $Y$ and $X(p)$, this is a much weaker requirement than global linearity, which can lead to low-bias estimates: for example, for the program $X(p) = \operatorname{Geo}(p)^3$ smoothed stochastic derivatives give a derivative estimate with $<0.5\%$ bias at $p = 0.01$, even though the cube function is highly non-linear on the inter-quartile range $[28, 137]$ of $\operatorname{Geo}(p)$. Smoothed stochastic derivatives recover the widely-used straight-through gradient estimator \cite{bengio} as a special case, as we work out in \cref{ex:st}.

%% file: demos.tex
We develop \texttt{StochasticAD.jl}, a prototype package for stochastic AD based on our theory of stochastic derivatives. As discussed in \cref{sec:smooth}, \emph{smoothed} stochastic derivatives obey the usual chain rule, and thus can be used with existing AD infrastructure by supplying custom rules for discrete random constructs, and we do so for a particle filter in \cref{sec:filter}. However, performing automatic differentiation with unsmoothed stochastic derivatives, which are unbiased in all cases, requires new innovation.
We develop a novel computational object called a \emph{stochastic triple}, introduced in \cref{sec:toy} and showcased in \cref{ex:walk} and \cref{sec:life}.

\vspace{-10px}
\subsection{Educational toy example of stochastic triples}
\vspace{-10px}
\label{sec:toy}
Forward-mode AD is often implemented with \emph{dual numbers}~\cite{baydin2018automatic}, which pair the primal evaluation of a deterministic function $f(p)$ with its derivative $\frac{\dd}{\dd p} f(p)$. Dual numbers can be propagated through a program using the chain rule. A useful alternative perspective of dual numbers is that they propagate an ``infinitesimal'' perturbation \texttt{ε} to the input through the program, where $\frac{\dd}{\dd p} f(p)$ is the coefficient of the ``dual'' element \texttt{ε}. 
Stochastic triples generalize dual numbers by including a third component $\Delta$s to describe finite perturbations with infinitesimal \emph{probability} (\cref{fig:toy}, left). 
\begin{figure} \centering
    \begin{tabular}{cc}
    \begin{minipage}[T]{0.5\linewidth}
        \begin{minted}[fontsize=\small, linenos, numbersep=4pt]{julia}
struct StochasticTriple
    value # primal evaluation
    δ # "infinitesimal" component
    Δs # component of discrete change
       # with "infinitesimal" 
       # probability
end
    \end{minted}
    \end{minipage} \hspace{-2em} \vline width .8pt & \hspace{0em}
    \begin{minipage}[T]{0.5\linewidth}
\begin{minted}[escapeinside=||, fontsize=\small, linenos, numbersep=4pt]{julia}
using Distributions
function X(p) |\colorbox{Cornsilk2}{p = 0.6 + ε}|
    a = p^2 |\colorbox{Cornsilk2}{0.36 + 1.2ε}|
    b = rand(Binomial(10, p)) 
    |\colorbox{Cornsilk2}{6 + (1 with probability 10.0ε)}|
    c = 2 * b + 3 * rand(Bernoulli(p))
    |\colorbox{Cornsilk2}{12 + (3 with probability 12.5ε)}| 
    return a * c * rand(Normal(b, a)) 
end
\end{minted}
    \end{minipage}
    \\
    \end{tabular}
    \begin{minted}[fontsize=\small,frame=lines, linenos, numbersep=2pt, escapeinside=||, ]{jlcon}
julia> using StochasticAD
julia> st = stochastic_triple(X, 0.6) # sample a single stochastic triple at p = 0.6
|\colorbox{Cornsilk2}{27.11 + 94.32ε + (6.78 with probability 12.5ε)}|
julia> derivative_contribution(st) # which produces a single derivative estimate...
|\colorbox{Cornsilk2}{179.04}|
julia> samples = [derivative_estimate(X, 0.6) for i in 1:1000] # take many estimates!
julia> println("d/dp of E[X(p)]: $(mean(samples)) ± $(std(samples) / sqrt(1000))")
|\colorbox{Gold2}{d/dp of E[X(p)]: 204.63 ± 1.25}|
    \end{minted}
    \caption{\textbf{Left:} Stochastic triple structure (simplified). \textbf{Right:} A toy program $X(p)$, using discrete distributions $\operatorname{Bin}(n,p)$, $\operatorname{Ber}(p)$, and the continuous normal distribution $\mathcal{N}(\mu, \sigma)$; the used stochastic derivatives are given in \cref{sec:derived}. Highlights show intermediate values during a single derivative estimate. 
    \textbf{Bottom:} Differentiating $\EE[X(p)]$; printout float precision reduced for clarity.}
    \label{fig:toy}
\end{figure}
 
In \cref{fig:toy}, we consider a toy program $X(p)$ including discrete randomness. We are interested in the derivative of $\EE[X(p)]$ at $p = 0.6$, and hence we provide the stochastic triple printed as \texttt{"0.6 + ε"} as input. First, the triple is squared and becomes \texttt{"0.36 + 1.2ε"}: this is the familiar way that dual numbers propagate, via the chain rule. But what happens when the triple is propagated through the discrete and random Binomial variable? The resultant stochastic triple \texttt{"6 + (1 with probability 10.0ε)"} is an integer with a component of \emph{discrete} change, reflecting an infinitesimal probability of one more success in the Binomial $\operatorname{Bin}(10, 0.6)$.
We can in fact understand why the probability is \texttt{"10ε"} by representing the Binomial as the sum of 10 Bernoulli variables, each with probability $0.6$. Since 4 of the Bernoulli's have an output of 0, they each have a probability \texttt{"2.5ε"} of switching to 1 (recall \cref{ex:ber}), and thus there is in total a \texttt{"10ε"} probability that the output of the Binomial increases  by 1 (rigorously, we have applied \cref{prop:composition}).
Formally, we can interpret a printout \texttt{"x + δε + (Δ with probability wε)"} as follows: \texttt{x} is a sample of the random variable $X(p)$ describing the primal evaluation, and \texttt{δ}, \texttt{w}, and \texttt{x+Δ} are samples of the components $\delta$, $w$, and $Y$, respectively, of the stochastic derivative of $X$ at $p$. 

Using our chain rule for stochastic derivatives (\cref{prop:composition}), we write rules for propagating stochastic triples through functions via operator overloading~\cite{RevelsLubinPapamarkou2016}, exploiting Julia's multiple dispatch feature~\cite{Julia-2017}. When multiple discrete changes are possible, we pick one probabilistically: we call this strategy \emph{pruning} (recall \cref{fig:path}) and show its unbiasedness in \cref{sec:pruning_proof}. For example, in line 6 of \cref{fig:toy}, right, we probabilistically choose between the perturbation to the Binomial and the perturbation to the Bernoulli, in this case picking the latter. To handle causal relationships between perturbations as in the first case of \cref{eq:cases}, we associate each perturbation with a tag to avoid erroneously pruning between two perturbations that occur simultaneously. Thus, stochastic triples can efficiently propagate through the full toy function written in \cref{fig:toy}. The function \texttt{derivative\_estimate} creates a stochastic triple, propagates it, and collapses it into the derivative estimate \texttt{δ + wΔ}, forming an unbiased estimate of the derivative via \cref{prop:unbiased} (\cref{fig:toy}, bottom).

\subsection{Inhomogeneous random walk}
\label{ex:walk}
\begin{figure}
\begin{subfigure}{1\textwidth}
    \refstepcounter{subfigure}\label{fig:walkcode}
    \refstepcounter{subfigure}\label{fig:walkgraph}
    \refstepcounter{subfigure}\label{fig:golboard}
    \end{subfigure}
    \includegraphics[width=\textwidth]{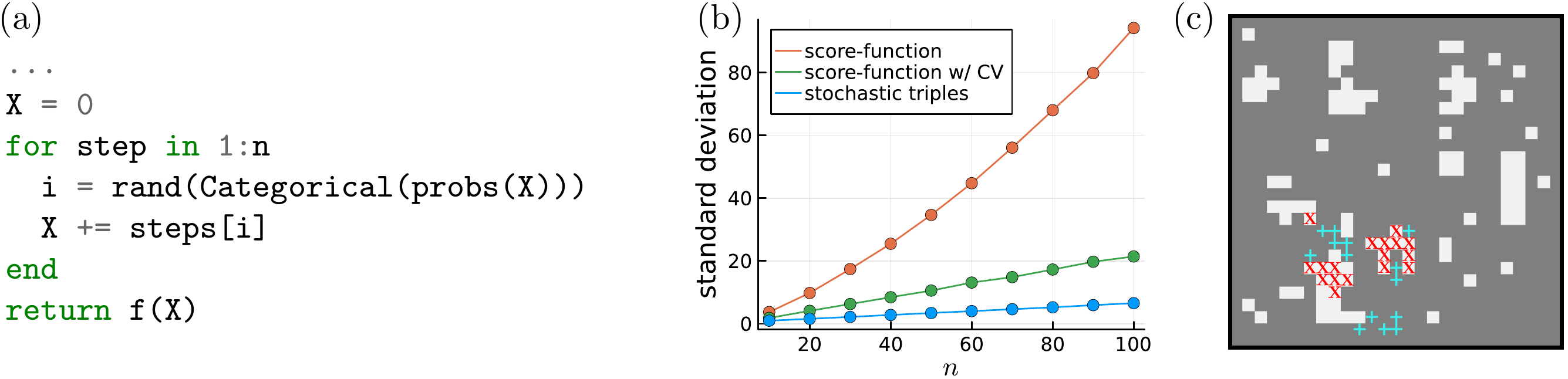}
    \caption{Automatic differentiation of discrete-time Markov processes. \textbf{(a)} Code snipped for a 1D random walk, which can be automatically differentiated by \texttt{StochasticAD.jl}; \texttt{probs(X)} gives the transition probabilities at value \texttt{X} and \texttt{steps[i]} gives the step size for the \texttt{i}th transition. \textbf{(b)} The variance of unbiased gradient estimates of the random walk program using stochastic triples and the score function, which is applied both without a control variate (CV) and with a pre-computed batch-average CV.  \textbf{(c)} The final board for one run of the stochastic Game of Life with $N = 25$ and $T = 10$, where the ``+'' signs represent additional living cells (white) in the stochastic alternative path, and the ``X'' signs represent additional dead cells (grey).}
    \label{fig:score}
\end{figure}

We consider a Markovian random walk $x_0, \dots, x_n$ on $\mathbb{Z}_{\geq 0}$, with transition behavior dependent on a parameter $p$ as follows,
 \begin{equation}
    x_n =\begin{cases} x_{n-1} + 1\textrm{ with probability } \exp\left(-\frac{x_{n-1}}{p}\right) \\
     x_{n-1} - 1\textrm{ with probability } 1 - \exp\left(-\frac{x_{n-1}}{p}\right)    \end{cases}, \quad x_0 = 0.
\end{equation}
We consider a program that stochastically simulates this walk and applies an arbitrary non-linear function $f$ to the output $x_n$. In practice, $f$ may represent a loss or a likelihood estimate; in this toy setting, we take $f(x) = x^2$. We are interested in studying the asymptotic behavior of the variance of our automatically derived gradient estimator, and so set $p = n$ so that the transition function varies appreciably over the range of the walk for all $n$. We find that the stochastic triple estimator has asymptotically lower variance than the score function estimator (\cref{fig:walkgraph}). 
Crucially, stochastic triples achieve this variance reduction while remaining entirely in discrete space, automatically producing a gradient estimate that is provably unbiased.

\subsection{Stochastic Game of Life}
\label{sec:life}

For our next example, we differentiate a stochastic version of John Conway's Game of Life, played on a two-dimensional board. In the traditional Game of Life, a dead cell becomes alive when 3 of its neighbors are alive, while a living cell survives when 2 or 3 of its neighbors are alive. In our stochastic version, each of these events instead has probability $95\%$, while their complementary events have probability $5\%$. Such types of discrete stochastic programs arise in many applications. For example, mixing machine learning with agent-based models found in epidemiology and sociological contexts~\cite{silverman2021situating,chopra2022differentiable} or rule-based models and discrete stochastic (Gillespie) simulations in systems biology~\cite{chylek2014rule,faeder2009rule} requires similar program constructs. 

Consider a program that populates each cell of an $N \times N$ board with probability $p$, runs the stochastic Game of Life for $T$ time steps, and counts the number of living cells $n_{\text{living}}$. We perform a sensitivity analysis of the final living population with respect to the initial living population, i.e. differentiate the expectation of $n_{\text{living}}$ with respect to $p$.
Stochastic triples propagate fully through the program, leading to an unbiased estimate of the derivative, as we verify with black-box finite differences. An example final board is depicted in \cref{fig:golboard}, along with the difference to the alternative final board chosen by pruning. This is a high-dimensional example (the state space has dimension $N^2 = 625$) with fundamentally discrete structure, providing support for the algorithmic correctness and generality of stochastic triples. In particular, the program cannot directly be continuously relaxed since it includes array indexing; in the approach of stochastic triples, integer-valued quantities stay integers.

\subsection{Particle filter}
\label{sec:filter}

\begin{figure}
\centering
\begin{subfigure}{1\textwidth}
    \refstepcounter{subfigure}\label{fig:pfviz}
    \refstepcounter{subfigure}\label{fig:pftail}
    \refstepcounter{subfigure}\label{fig:pfscaling}
    \end{subfigure}
\includegraphics[width=1\textwidth]{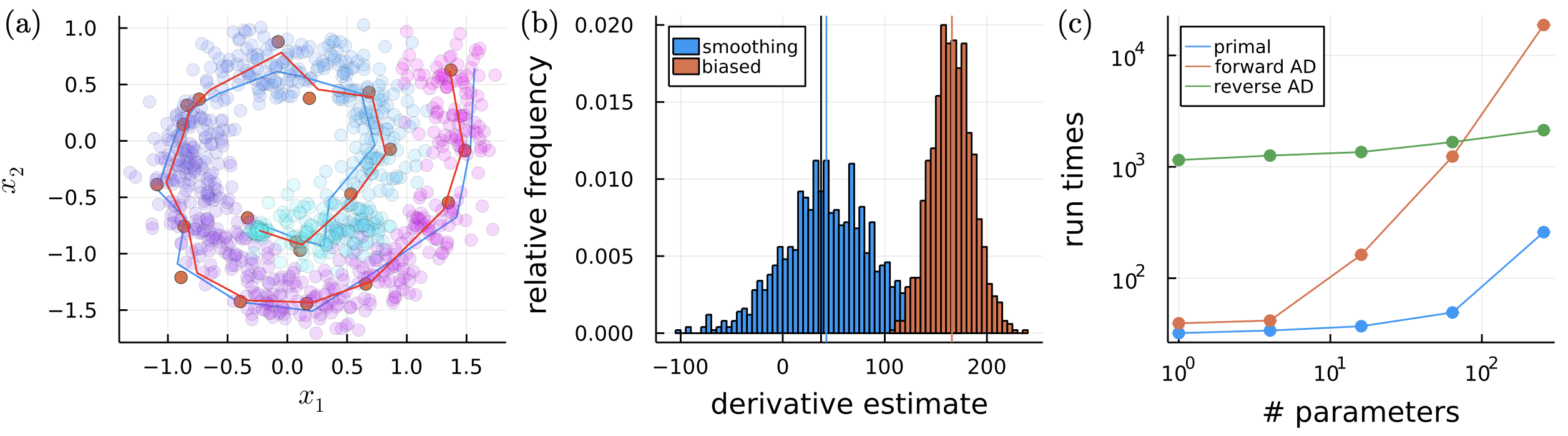}
\caption{
Unbiased differentiable particle sampler using smoothed stochastic derivatives. 
\textbf{(a)} Latent path (blue) and observations (orange) for a particular realization of the process given by the hidden Markov model (see \cref{sec:SI_filter} for details) in $d=2$ dimensions, including the recursively applied Kalman filter estimator (red) and the particles of the particle filter (going from blue at $n=1$ to pink at $n=20$).   
\textbf{(b)} The value of the derivative of $\log \mathcal{L}$ with respect to the first parameter for $d=2$ calculated by differentiating the Kalman filter is marked by a black line. The blue distribution corresponds to the derivative computed with smoothed stochastic derivatives. The orange distribution represents the biased derivative approach, where the resampling step is not differentiated. The means of the two distributions are highlighted by a line in the respective color. In both cases, we use 1000 samples. 
\textbf{(c)} Averaged run times in $\mathrm{ms}$ comparing forward- and reverse-mode AD for increasing numbers of parameters. 
}
\label{fig:ps}
\end{figure}

As a final example, we consider a hidden Markov model with random latent states $X_1, \dots, X_n$, observations $y_1 \sim Y_1 , \dots, y_n \sim Y_n$, and parameters $\theta$.  
The likelihood $\mathcal{L}= p(y_1, \dots, y_n \mid \theta)$ is in general not tractable, but a particle filter can be used to compute an estimate of $\mathcal{L}$.
Here, we assume familiarity with particle filters \cite{chopin2020introduction, doucet2009tutorial} but provide a short exposition with focus on the resampling step and experimental details in \cref{sec:SI_filter}. We assume the latent states to be continuous random, but discrete randomness enters through the resampling step.

To differentiate the particle filter resampling step, we provide a stochastic derivative formulation of the returned particles and importance weights (\cref{sec:SI_filter}).
Importantly, $\mathcal{L}$ can be expressed by the sum of the weights at the last step, and the weight assigned to each particle in the resampling steps is used in a \emph{purely linear} way. 
Smoothed stochastic derivatives permit \emph{unbiased} reverse-mode AD in this case. Our formulated approach, though derived very differently, is equivalent to the particle filter AD scheme developed in \cite{scibior2021differentiable} implementing the first estimator derived in \cite{poyiadjis2011particle}, as we show in \cref{sec:SI_filter}. Our particular choice of system allows the calculation of a ground-truth gradient of $\mathcal{L}$ by differentiating the Kalman filter algorithm~\cite{kalman1960new}.
\cref{fig:pfviz} visualizes the latent process and observations, and the Kalman and particle filter trajectories. 
Our estimator agrees with the Kalman filter derivative, unlike biased estimators~\cite{jonschkowski2018differentiable,corenflos2021differentiable,rosato2021efficient} that neglect the contribution of the resampling step or perform it with entropy-regularized optimal transport~(\cref{fig:pftail}). For our program, we observe reverse-mode AD to perform better for more than $\approx 100$ parameters (\cref{fig:pfscaling}). However, we find that the variance increases more rapidly with the number of steps and dimension as compared to biased estimators, suggesting that there is room for improvement in the coupling approach~\cite{scibior2021differentiable, poyiadjis2011particle}. 

%% file: outlook.tex
We have presented a method for unbiased AD of programs with discrete randomness, which, we have argued, is a natural generalization of pathwise gradient estimators based on the reparameterization trick to the discrete case. 
However, more work will need to be done to turn our software demonstration \texttt{StochasticAD.jl} into an AD system capable of handling the full complexity of applications such as machine learning. 
A useful improvement would be better support for discrete constructs such as \texttt{if} statements with discrete random input; currently, such branches need to be rewritten using array indexing, which is supported. Further, an interesting direction for future work is to automatically handle functions which deterministically turn a continuous random quantity into a discrete one, such as an inequality comparison \texttt{X(p) > 0} or a Bernoulli variable implemented implicitly as \texttt{rand() < p} and also to handle constructs such as while loops based on such functions. It would also be interesting to explore synergies of our approach with ADEV~\cite{lew2022adev}, a Haskell-based framework for provably correct stochastic AD developed concurrently with our work, which formulates gradient estimation strategies using Haskell's continuation passing style.
\looseness-1


We also expect future work to focus on further variance reduction. Our method's variance depends on the degree to which the primal and derivative computations can be \emph{coupled}: while we present a natural method of coupling for a number of elementary stochastic programs and their compositions, the design space is rich when it comes to challenging examples such as the Game of Life or the resampling step of a particle filter.\ (This design space is reflected in the fact that the form of the stochastic derivative depends not only on a program's probability distribution, but also on the way it is parameterized.) Furthermore, the pruning operation can introduce additional variance: for example, the derivative estimate automatically produced for the program $B_1 + 2B_2$ with i.i.d. $B_i \sim \operatorname{Bin}(n, p)$ has variance $\mathcal{O}(n^2)$ due to the pruning between the \texttt{+1} and \texttt{+2} perturbations, even though stochastic derivatives give gradient estimators with variance $\mathcal{O}(n)$ for each of $B_1$ and $B_2$. Smoothing does not face this issue but accrues a bias through non-linear functions. Pruning and smoothing may be thought of as the simplest ways to construct an AD algorithm from stochastic derivatives, lying on opposite ends of the design space: we anticipate future work to address their suboptimalities and ideally form unbiased estimators for discrete random programs that fully close the variance gap to their continuous counterparts~\cite{kingma2013auto,maddison2016concrete}. Finally, going beyond smoothing for reverse-mode AD, and ideally achieving unbiasedness while remaining coupled to the primal, is an important open problem for large-scale applications.
\looseness=-1

%% file: checklist.tex
\section*{Checklist}


\begin{enumerate}

\item For all authors...
\begin{enumerate}
  \item Do the main claims made in the abstract and introduction accurately reflect the paper's contributions and scope?
    \answerYes{}
  \item Did you describe the limitations of your work?
    \answerYes{} See \cref{sec:outlook}.
  \item Did you discuss any potential negative societal impacts of your work?
    \answerNo{}
  \item Have you read the ethics review guidelines and ensured that your paper conforms to them?
    \answerYes{}
\end{enumerate}

\item If you are including theoretical results...
\begin{enumerate}
  \item Did you state the full set of assumptions of all theoretical results?
    \answerYes{} See Appendix.
        \item Did you include complete proofs of all theoretical results?
    \answerYes{} See Appendix.
\end{enumerate}

\item If you ran experiments...
\begin{enumerate}
  \item Did you include the code, data, and instructions needed to reproduce the main experimental results (either in the supplemental material or as a URL)?
    \answerYes{} See \cref{sec:implementation} and the provided code. 
  \item Did you specify all the training details (e.g., data splits, hyperparameters, how they were chosen)?
     \answerYes{} 
    \item Did you report error bars (e.g., with respect to the random seed after running experiments multiple times)?
    \answerNo{} Our examples were small-scale, so we could take enough samples to achieve minimal error. 
     \item Did you include the total amount of compute and the type of resources used (e.g., type of GPUs, internal cluster, or cloud provider)?
    \answerYes{} See \cref{sec:implementation}.
\end{enumerate}

\item If you are using existing assets (e.g., code, data, models) or curating/releasing new assets...
\begin{enumerate}
  \item If your work uses existing assets, did you cite the creators?
    \answerYes{} We cite the Julia packages on which this work is based. 
  \item Did you mention the license of the assets?
    \answerNo{} We only use open-source software packages.
  \item Did you include any new assets either in the supplemental material or as a URL?
    \answerYes{} We have included our package in the supplemental material.
  \item Did you discuss whether and how consent was obtained from people whose data you're using/curating?
    \answerNA{} We do not use any previously collected data.
  \item Did you discuss whether the data you are using/curating contains personally identifiable information or offensive content?
    \answerNA{}
\end{enumerate}

\item If you used crowdsourcing or conducted research with human subjects...
\begin{enumerate}
  \item Did you include the full text of instructions given to participants and screenshots, if applicable?
    \answerNA{}
  \item Did you describe any potential participant risks, with links to Institutional Review Board (IRB) approvals, if applicable?
    \answerNA{}
  \item Did you include the estimated hourly wage paid to participants and the total amount spent on participant compensation?
    \answerNA{}
\end{enumerate}

\end{enumerate}

%% file: derived.tex
\section{Examples of stochastic derivatives}
\label{sec:derived}

In this section, we present examples of stochastic derivatives for a number of stochastic programs. 

\begin{example}[Integer-valued stochastic programs]
\label{ex:int}

\begin{table}
  \caption{Probability mass functions and cumulative distributions of integer-valued stochastic programs, where $x$ is assumed in the domain of $X(p)$.}
  \label{mass-table}
  \centering
  \begin{tabular}{lllll}
    \toprule
    $X(p)$ & Parameter & $\PP(X(p) = x)$ & $\PP(X(p) \leq x)$ & $\partial_p \PP(X(p) \leq x)$ \\
    \midrule
    $\operatorname{Ber}(p)$ & Probability & $px + (1-p)(1-x)$ & $(1-p) + px$ & $x-1$  \\
    $\operatorname{Geo}(p)$ & Probability & $p(1-p)^x$ & $1-(1-p)^{x}$ & $x(1-p)^{x-1}$ \\
    $\operatorname{Pois}(p)$ & Rate & $p^x/(e^p x!)$ & $\sum_{k=0}^{x} e^{-p} p^k/k!$ & $-p^{x+1}/(e^p (x+1)!)$\\
    \bottomrule
  \end{tabular}
\end{table}
The Bernoulli, Binomial, Geometric, and Poisson distributions are all discrete, and in fact (nonnegative) integer-valued. Suppose $X(p)$ follows one of these distributions, parameterized using the inversion method~\cite{devroye2006nonuniform} over the interval $[0,1]$. We present a worked derivation of the stochastic derivative of $X(p)$, following the construction used in \cref{thm:main}.

We construct a stochastic derivative $(\delta, w, Y)$ at input $p$ satisfying \cref{def:deriv}. Since $X(p)$ is parameterized via the inversion method, $\dd X(\varepsilon) \in \{-1, 0, 1\}$ for small enough $\varepsilon$. Thus, for small enough $\varepsilon$, the conditional distribution of $\dd X (\varepsilon)$ given $X(p) = x$ is:
\begin{align}
\PP\left(\dd X(\varepsilon) = 1 \xmiddle| X(p) = x \right) &= \frac{\PP(X(p+\varepsilon) = x+1\text{ and }X(p) = x)}{\PP(X(p) = x)}, \\ 
\PP\left(\dd X(\varepsilon) = -1 \xmiddle| X(p) = x \right) &= \frac{\PP(X(p+\varepsilon) = x-1\text{ and }X(p) = x)}{\PP(X(p) = x)},
\end{align}
and $\dd X(\varepsilon) = 0$ otherwise. According to the inversion method, $X(p)(\omega) = x$ for $\omega$ between $\PP(X(p) \leq x - 1)$ and $\PP(X(p) \leq x)$. Taking the derivative of the above quantities as $\varepsilon \to 0$ therefore yields:
\begin{align}
w_+ &:= \lim_{\varepsilon \to 0} \frac{\PP\left(\dd X(\varepsilon) = 1 \xmiddle| X(p) = x \right)}{\varepsilon}, \\
&= \frac{\frac{\dd}{\dd p} \PP(X(p) \leq x)}{\PP(X(p) = x)} \ind\{\varepsilon \cdot \partial_p \PP(X(p) \leq x) < 0\},
\end{align}
and
\begin{align}
w_- &:= \lim_{\varepsilon \to 0} \frac{\PP\left(\dd X(\varepsilon) = -1 \xmiddle| X(p) = x \right)}{\varepsilon}, \\
&= \frac{\frac{\dd}{\dd p} \PP(X(p) \leq x - 1)}{\PP(X(p) = x)} \ind\{\varepsilon \cdot \partial_p \PP(X(p) \leq x -1) > 0\},
\end{align}
where the $\varepsilon$ included on the right hand sides is an abuse of notation to express their dependence on the direction of the limit, and $\ind\{C\} = \ind_{\{C\}}$ for a condition $C$. By the differentiability of $\PP(\dd X(\varepsilon) \neq 0)$, $\dd X(\varepsilon)/\varepsilon$ almost surely approaches $0$, so $\delta = 0$. A simple choice for the bound $B$ is~$1$: then $|\dd X(\varepsilon)| > B\varepsilon = \epsilon \iff \dd X(\varepsilon) \in \{-1, 1\}$. We may now form $w$ as the almost sure limit of $\PP\left(A_B(\varepsilon) \xmiddle| X(p)\right) / \varepsilon = \PP\left(\dd X(\varepsilon) \neq 0 \xmiddle| X(p)\right)$ and $Y$ as the limit of the conditional distribution of $X(p) + \dd X(\varepsilon)$ given $X(p)$ and $\dd X(\varepsilon) \neq 0$. By the above, $w$ is given conditionally on $X(p) = x$ as
\begin{equation}
w = w_+ + w_-,
\end{equation}
while the probability distribution of $Y$ is given conditionally on $X(p) = x$ as
\begin{align}
	\PP\left(Y = x+1 \xmiddle| X(p) = x\right) &= \frac{w_+}{w_+ + w_-}, \\
	\PP\left(Y = x-1 \xmiddle| X(p) = x\right) &= \frac{w_-}{w_+ + w_-}.
\end{align}

\begin{table}
  \caption{Stochastic derivatives of integer-valued stochastic programs}
  \label{derived-table}
  \centering
  \begin{tabular}{lll}
    \toprule
    &
    \multicolumn{2}{c}{Conditionally on $X(p) = x$}                   \\
    \cmidrule(r){2-3}
      $X(p)$  & $w_-$ & $w_+$ \\
    \midrule
      $\operatorname{Ber}(p)$ & $-1/p\cdot\ind\{x = 1\text{ and }\varepsilon < 0\}$ & $1/(1-p) \cdot \ind\{x = 0\text{ and }\varepsilon > 0\}$  \\
      $\operatorname{Bin}(n, p)$ & $-x/p \cdot \ind\{x > 0\text{ and }\varepsilon < 0\} $& $(n-x)/(1-p) \cdot \ind\{x < n\text{ and }\varepsilon > 0\}$  \\
      $\operatorname{Geo}(p)$ & $x/\left(p(1-p)\right) \cdot \ind\{x > 0\text{ and }\varepsilon > 0\}$ & $-(x+1)/p \cdot\ind\{\varepsilon < 0\}$  \\
      $\operatorname{Pois}(p)$ & $-x/p \cdot \ind\{x > 0\text{ and }\varepsilon < 0\}$ & $\ind\{\varepsilon > 0\}$ \\
    \bottomrule
  \end{tabular}
\end{table}
\cref{derived-table} lists the weights $w_+$ and $w_-$ for the Bernoulli, Binomial, Poisson, and Geometric distributions, derived using the parameterizations given in \cref{mass-table}. Using $w_+$ and $w_-$, we can easily express $w$ and $Y$ as above. Note that $w_+$ and $w_-$ for a Binomial variable $X(p) \sim \operatorname{Bin}(n,p)$ may be derived via representation as the sum of Bernoulli variables, with the primal evaluation $X(p) = x$ corresponding to the sum of $x$ Bernoulli successes and $n - x$ Bernoulli failures. 

The stochastic derivative for a Binomial variable $\operatorname{Bin}(n,p)$ has weight $w$ of order $n$ and $Y$ deterministically either $1$ or $-1$ depending on the chosen direction of the derivative. Therefore, the derivative estimator $w (Y - X(p))$ has variance of order $n$. In contrast, the score method gives an estimator
\begin{equation}
   \frac{X(p)(X(p) - np)}{ p(1 - p)} 
\end{equation}
which is also unbiased, but its variance
\begin{equation}
n^3 \frac{p}{1-p} + \mathcal{O}(n^2)
\end{equation}
is growing with cubic rate in $n$.

\end{example}

\begin{example}[Differentiably parameterized continuous stochastic programs]
\label{ex:repar}
	For continuous stochastic programs, our method reduces to the standard reparameterization trick. Specifically, for a differentiably parameterized continuous stochastic program $X(p)$, the stochastic derivative w.r.t a parameter $p$ of the distribution has the form $(\delta, 0, 0)$. To satisfy \cref{def:deriv}, a natural choice of bound is $B = |\delta| + 1$, so that $A_B^c(\varepsilon) = \{|\dd X(\varepsilon)| \leq \varepsilon |\delta| + \varepsilon\}$ simply requires that the higher-order correction to the derivative linearization be smaller than $\varepsilon$ when applied at a point at distance $\varepsilon$ from $p$. 
	
	For instance, the stochastic derivative of $X(p) \sim \operatorname{Exp}(p)$ w.r.t. $p$, the scale of the Exponential distribution, is
\begin{equation}
\left(\frac{x}{p}, 0, 0\right)    
\end{equation}
conditionally on $X(p) = x$. Such rules are already implemented in the Julia AD ecosystem in \texttt{Distributions.jl}~\citeSI{Distributions.jl-2019} and \texttt{DistributionsAD.jl}~\citeSI{JSSv098i16}, and may be used without modification as the continuous case is a special case of our formalism.
\end{example}

\begin{example}[Categorical variable]
Consider a categorical variable $\operatorname{Categorical}_{\mathbf{a}}(p_1, p_2, p_3, \dots, p_n)$ assuming fixed outputs $\mathbf{a} = a_1, a_2, \dots, a_n$ with probabilities $p_1, p_2, p_3, \dots, p_n > 0$. Assume that $p_1, p_2, p_3, \dots, p_n$ implicitly depend on a single parameter $p$, to fit our definition of stochastic derivative w.r.t. a single parameter, and let $X(p) \sim \operatorname{Categorical}_{\mathbf{a}}(p_1, p_2, p_3, \dots, p_n)$. 

Suppose $X(p)$ is parameterized via the inversion method with outputs ordered as $a_1, a_2, \dots a_n$.
Similar logic to \cref{ex:int} then yields that its stochastic derivative is of the form $(0, w, Y)$ with $Y \in \{a_{x-1}, a_{x+1}\}$, and
\begin{align}
   w \cdot \PP\left(Y = a_{x+1}\right) &= w_+ = \frac{\left|\sum_{i=1}^{x} \partial_p p_i\right|}{p_x} \cdot \ind\left\{\varepsilon \cdot \sum_{i=1}^x \partial_p p_i < 0\right\} , \\
   w \cdot \PP\left(Y = a_{x-1}\right) &= w_- = \frac{\left|\sum_{i=1}^{x-1} \partial_p p_i\right|}{p_x} \cdot \ind\left\{\varepsilon \cdot \sum_{i=1}^{x-1} \partial_p p_i > 0\right\}.
\end{align}
In theory, $a_1, a_2, \dots, a_n$ can be arbitrary discrete objects such as arrays or strings, as the categorical variable could represent an intermediate value of the program that is ultimately converted into a number. (Interpreted strictly, our formalism requires an embedding into Euclidean space; one can imagine a trivial embedding $a_i \mapsto i$ which is then provided to a learned embedding.) In handling such cases, one must be careful to speak in terms of \emph{alternate} outputs rather than perturbations to the output. Our formalism of stochastic derivatives can already accommodate this idea, as $Y$ takes on the alternate values rather than the values of the perturbations.

\end{example}
We now consider some general stochastic programs. Their stochastic derivatives arise automatically from composition and need not be hand-derived, but we do so to illustrate \cref{prop:composition}.

\begin{example}[Bernoulli plus Exponential]
Consider the program
\begin{align}
    X_1(p) &\sim \operatorname{Ber}(p), \\ 
    X_2(p) &\sim \operatorname{Exp}(p), \\
    X(p) &= X_1(p) + X_2(p) \sim \operatorname{Ber}(p) + \operatorname{Exp}(p).
\end{align}
Then, $X(p)$ has right (left) stochastic derivative,
\begin{equation}
   \left(\frac{x_2}{p}, w_1, Y_1\right) 
\end{equation}
conditionally on $X_1(p) = x_1$ and $X_2(p) = x_2$, where $w_1$ and $Y_1$ are given conditionally on the Bernoulli output $X_1(p) = x_1$ for $\varepsilon > 0$ ($\varepsilon < 0$) by \cref{derived-table}.
\end{example}

\begin{example}[Cubing a geometric]
Consider the program
\begin{align}
    X_1(p) &\sim \operatorname{Geo}(p), \\
    X(p) &= X_1(p)^3.
\end{align}
We know from \cref{ex:int} that $X_1(p)$ has right (left) stochastic derivative $(0, w_1, Y_1)$ with $w_1$ and $Y_1$ given by \cref{derived-table} for $\varepsilon > 0$ ($\varepsilon < 0$). In particular, $Y_1 \in \{x_1 -1, x_1+1\}$ conditionally on $X_1(p) = x_1$. Thus, intuitively, when propagating through the cube function we only care about the discretely spaced alternate values $(x_1-1)^3$ and $(x_1+1)^3$; the conventional derivative $3x_1^2$ is irrelevant as the input to the cube function is integer-valued. Indeed, by \cref{prop:composition} the stochastic derivative of $X(p)$ reads
\begin{equation}
   (0, w_1, Y_1^3), 
\end{equation}
where $Y_1^3 \in \{(x_1-1)^3, (x_1+1)^3\}$.
\end{example}

\begin{example}[Parameter-scaled Bernoulli]
Consider the program,
\begin{align}
    X_1(p) &\sim \operatorname{Ber}(p), \\
    X(p) &= p \cdot X_1(p).
\end{align}
By \cref{prop:composition}, the stochastic derivative of $X(p)$ reads,
\begin{equation}
	(X_1(p), w_1, pY_1),
\end{equation}
	where $w_1$ and $Y_1$ are given for $X_1(p)$ by \cref{derived-table}.
\end{example}

\begin{example}[Two-step random walk]
Consider the program,
\begin{align}
    X_1(p) &\sim \operatorname{Ber}(p), \\
    X_2(x_1) &\sim 
    \begin{cases}
    \operatorname{Ber}(p) & \text{if }x_1 = 0 \\
    \operatorname{Ber}(2p) & \text{if } x_1 = 1.
    \end{cases} \\
    X(p) &= X_1(p) + (X_2 \circ X_1)(p).
\end{align}
This represents a two-step random walk, where there is a transition $0 \to 1$ with probability $p$, a transition $1 \to 2$ with probability $2p$, and self loops otherwise. In this example, let us focus on a particular primal evaluation where $X_1(p) = 0$ and $X_2(X_1(p)) = 0$, so that $X(p) = 0$, and consider only the right stochastic derivative for simplicity. Conditionally on $X_1(p) = 0$, $X_1(p)$ has a right stochastic derivative $(\delta_1, w_1, Y_1)$ given by
\begin{equation}
    \left(0, \frac{1}{1-p}, 1\right),
\end{equation}
using \cref{derived-table}. Similarly, conditionally on $X_2(X_1(p)) = X_2(0) = 0$, $X_2$ has right stochastic derivative $(\delta_2, w_2, Y_2)$ given by
\begin{equation}
    \left(0, \frac{1}{1-p}, 1\right).
\end{equation}
Let us now turn our attention to the stacked program $[X_1(p); (X_2 \circ X_1)(p)]$. By \cref{prop:composition}, the stacked program has right stochastic derivative $(0, w, Y_{12})$ where conditionally on $X_1(p) = 0$ and $X_2(0) = 0$,
\begin{equation}
    Y_{12} \quad =
    \quad \begin{cases}
    [1; X_2(1)] & \text{ with probability }\quad 1/2, \\
    [0; 1] & \text{ with probability }\quad 1/2,
    \label{eq:cases12a}
    \end{cases}
\end{equation}
and $w = \frac{2}{1-p}$. Note that conditionally on $X_2(0) = 0$, $X_2(1)$ has a chance $\frac{1-2p}{1-p}$ of also being 0, but a $\frac{p}{1-p}$ chance of flipping to 1, so that the first case expands into two events
\begin{equation}
    Y_{12} \quad =
    \quad \begin{cases}
    [1; 1] & \text{ with probability }\quad \frac{p}{2(1-p)}, \\
    [1; 0] & \text{ with probability }\quad \frac{1-2p}{2(1-p)}, \\
    [0; 1] & \text{ with probability }\quad 1/2.
    \label{eq:cases12b}
    \end{cases}
\end{equation}
Finally, the expression for $X(p)$ may be thought of as a unary sum function operating on this stacked program, with stochastic derivative $(0, w, Y)$ where conditionally on $X_1(p) = 0$ and $X_2(0) = 0$ 
\begin{equation}
    Y \quad =
    \quad \begin{cases}
    2 & \text{ with probability }\quad \frac{p}{2(1-p)}, \\
    1 & \text{ with probability }\quad \frac{1-2p}{2(1-p)} + 1/2.
    \label{eq:cases12c}
    \end{cases}
\end{equation}
\label{ex:linking}
\end{example}

Finally, we present an example using smoothed stochastic derivatives to rederive a popular gradient estimator.

\begin{example}[Recovering the straight-through gradient estimator of \cite{bengio}]
\label{ex:st}
The straight-through gradient estimator formally assigns a derivative of 1 to the hard-thresholding function $\operatorname{HT} = \ind_{[0,\infty)}$, such that a Bernoulli variable $\operatorname{Ber}(p) \sim \operatorname{HT}(p - U[0,1])$ is assigned a derivative of 1. Now, let $\tilde{\delta}_L$ and $\tilde{\delta}_R$ be the left and right smoothed stochastic derivatives of $\operatorname{Ber}(p)$. By \cref{derived-table} and \cref{def:smooth}, for the right-sided ($\varepsilon > 0$) case,
\begin{equation}
    \tilde{\delta_R} = 1/(1-p) \cdot \ind\{x = 0\},
\end{equation}
and for the left-sided ($\varepsilon < 0$) case,
\begin{equation}
    \tilde{\delta_L} = 1/p \cdot \ind\{x = 1\}.
\end{equation}
By linearity of expectation applied to \cref{def:smooth}, any affine combination of $\tilde{\delta_L}$ and $\tilde{\delta_R}$ is also a valid smoothed stochastic derivative. In particular,
\begin{equation}
   1 = (1 - p) \cdot \tilde{\delta}_L + p \cdot \tilde{\delta}_R,
   \label{eq:affine}
\end{equation}
is a valid smoothed stochastic derivative of a Bernoulli variable, which explains why the straight-through estimator provides a low-bias estimate. Note that a constant-valued affine combination of the left and right smoothed stochastic derivatives is not possible in general, e.g. it is not possible for Geometric and Poisson random variables, in which cases smoothed stochastic derivatives generalize the straight-through estimator (mixing $\tilde{\delta}_R$ and $\tilde{\delta}_L$ may still be useful in these cases to reduce bias.) 

\end{example}

%% file: proofs.tex
\section{Proofs}
\label{sec:proofs}

\subsection{Preliminaries}
Recall from the main text,
\newtheorem*{defcopy2}{\cref{def:program}}
\begin{defcopy2}
A \emph{stochastic program} $X(p)$ is a stochastic process with values in a Euclidean space $E$, whose index set $I$ is either an open subset of a Euclidean space or a closed real interval.\end{defcopy2}
Throughout the formalism, we let $X(p)$ denote such a stochastic program, where $I = [a,b] \subset \mathbb{R}$ is a closed interval. As noted in the main text, it is sufficient to consider this case because sensitivities of stochastic programs $Z$ with more general index sets can be understood for an input $\mathbf{u}$ by studying at $p = 0$ the directional perturbation $X(p) = Z(\mathbf{u} + p \mathbf{v})$ in a direction $\mathbf{v}$, where $X(p)$ is then a stochastic program with index set a closed interval containing 0.

As in the main text, we use the shorthand
\begin{equation}
    \dd X(\varepsilon) = X(p + \varepsilon) - X(p).
\end{equation}
A number of statements and propositions have identical forms for right and left stochastic derivatives, only differing in the direction of the limits. To accomodate this, we often use the notation $\varepsilon \to 0^{+/-}$ to indicate that statements and proofs for right and left stochastic derivatives can both be read off by choosing the appropriate side of the limit for all expressions, where objects such as $w$ have different definitions depending on the chosen reading  $\epsilon \to 0^+$ or $\epsilon \to 0^-$.

For clarity, we remark upon the (standard) notation,\begin{equation}
    W = \EE\left[U \xmiddle| V\right],
\end{equation}
where $U$ and $V$ are both random variables defined on a sample space $\Omega$. In this setting, $W$ is itself also a random variable defined on $\Omega$. In particular, for a particular sample $\omega \in \Omega$, $W(\omega)$ is the conditional expectation $\EE\left[U \xmiddle| V = V(\omega)\right]$. In the language of $\sigma$-algebras, $\EE\left[U \xmiddle| V\right]$ is equivalent to $\EE\left[U \xmiddle| \sigma(V)\right]$, where $\sigma(V)$ is the sub-$\sigma$-algebra generated by $V$.

\subsection{Unbiasedness of stochastic derivatives}

We first prove a general result regarding when a class of events $A(\varepsilon)$ allows for the pathwise gradient estimator $\delta$ to be applied when the probability space is restricted to $A^c(\varepsilon)$.

\begin{proposition}\label{prop:repara}
Suppose $\dd X(\varepsilon)/\varepsilon \to \delta$ almost surely, and $|\dd X(\varepsilon)| \leq B |\varepsilon|$ holds given the event $A^c(\varepsilon)$, where $B > |\delta|$ is integrable. 
Then, for any function $f\colon E \to \mathbb{R}$ with bounded derivative,  
\begin{equation}
\lim_{\varepsilon \to 0^{+/-}}\EE \left[\frac{f(X(p+\varepsilon)) - f(X(p))}{\varepsilon} \ind_{A^c(\varepsilon)} \xmiddle| X(p)\right] = \EE\left[f'(X(p))\delta \xmiddle| X(p)\right].
\end{equation}
\end{proposition}
\begin{proof}

Since $\dd X(\varepsilon)/\varepsilon \to \delta$ almost surely and $B > |\delta|$, $|\dd X(\varepsilon)| \leq B|\varepsilon|$ holds almost surely as $\varepsilon \to 0$. So $\ind_{A^c(\varepsilon)} \to 1$ almost surely as $\varepsilon \to 0$. By the chain rule,
\begin{equation}
\frac{f(X(p+\varepsilon)) - f(X(p))}{\varepsilon}\ind_{A^c(\varepsilon)} \to f'(X(p))\delta.
\end{equation}
almost surely as $\varepsilon \to 0^{+/-}$.
Furthermore,
\begin{align}
\left|\frac{f(X(p+\varepsilon)) - f(X(p))}{\varepsilon} \ind_{A^c(\varepsilon)}\right|
&\le \frac1\varepsilon{\|f'\|_\infty|\dd X(\varepsilon)|} \ind_{A^c(\varepsilon)} \\
&\leq B \|f'\|_\infty =: G.
\end{align}
By assumption $\EE\, G < \infty.$ 
The proposition follows by the dominated convergence theorem for conditional expectations.
\end{proof}
As in the main text, define
\begin{equation}\label{eq:finitechange}
    A_B(\varepsilon) =  \{|\dd X(\varepsilon)| > B|\varepsilon|\}\end{equation}
and recall
\newtheorem*{defcopy}{\cref{def:deriv}}
\begin{defcopy}[Stochastic derivative]
Suppose $X(p) \in E$ is a stochastic program with index set $I$ a closed real interval. 
We say that the triple of random variables $(\delta, w, Y)$, with $w \in \mathbb{R}$ and $Y \in E$, is a right (left) \emph{stochastic derivative} of $X$ at the input $p \in I$ if $\dd X(\varepsilon) / \varepsilon \to \delta$ almost surely as $\varepsilon \to 0$, and there is an integrable random variable $B > |\delta|$ such that for all bounded functions $f\colon E \to \mathbb{R}$ with bounded derivative it holds almost surely that
\begin{equation}
\tag{\ref{eq:deriv}}
\EE\left[w \left(f(Y) - f(X(p))\right) \xmiddle| X(p) \right] = \lim_{\varepsilon \to 0^{+/-}} \EE\left[\frac{f(X(p + \varepsilon)) - f(X(p))}{\varepsilon} \ind_{A_B(\varepsilon)} \xmiddle| X(p)\right],
\end{equation}
with limit taken from above (below), where $\PP\left(A_B(\varepsilon) \xmiddle| X(p)\right)/\varepsilon$ is dominated by an integrable random variable for all $\varepsilon > 0$ ($\varepsilon < 0$). 
\end{defcopy}

Using \cref{prop:repara}, we may show that outside the event $A_B(\varepsilon)$, the sensitivity is well-described by the pathwise gradient estimator, allowing us to prove \cref{prop:unbiased}.
\newtheorem*{propcopy2}{\cref{prop:unbiased}}
\begin{propcopy2}[Unbiasedness]
If $(\delta, w, Y)$ is a stochastic derivative of $X(p)$ at $p$, it holds that
\begin{equation}
\tag{\ref{eq:unbiased}}
\frac{\dd\EE \left[X(p)\right]}{\dd p} = \EE[\delta + w(Y-X(p))].
\end{equation}
\end{propcopy2}
\begin{proof}
Let $B$ be the associated bound. Since $\dd X(\varepsilon) / \varepsilon \to \delta$ almost surely as $\varepsilon \to 0$ and $A_B^c(\varepsilon)$ implies ${\dd X(\varepsilon) / \varepsilon \leq B\varepsilon}$, \cref{prop:repara} applied to $A_B(\varepsilon)$ with $f$ as identity implies
\begin{align}
    \label{eq:pieceone}
    \lim_{\varepsilon \to 0^{+/-}}\EE\left[\frac{X(p + \varepsilon) - X(p)}{\varepsilon} \ind_{A_B^c(\varepsilon)} \xmiddle| X(p)\right] = \EE\left[\delta \xmiddle| X(p)\right],
\end{align}
while \cref{eq:deriv} of \cref{def:deriv} also applied with $f$ as identity gives
\begin{align}
    \label{eq:piecetwo}
    \lim_{\varepsilon \to 0^{+/-}}\EE\left[\frac{X(p + \varepsilon) - X(p)}{\varepsilon} \ind_{A_B(\varepsilon)} \xmiddle| X(p)\right] = \EE\left[w(Y - X(p)) \xmiddle| X(p)\right].
\end{align}
Summing \cref{eq:pieceone} and \cref{eq:piecetwo}, we have
\begin{align}
    \lim_{\varepsilon \to 0^{+/-}}\EE\left[\frac{X(p + \varepsilon) - X(p)}{\varepsilon} \xmiddle| X(p)\right] &= \EE\left[\delta + w(Y - X(p)) \xmiddle| X(p)\right],
\end{align}
and applying the tower property of conditional expectations, we obtain
\begin{equation}
    \frac{\dd \EE[X(p)]}{\dd p} = \lim_{\varepsilon \to 0^{+/-}}\EE\left[\frac{X(p + \varepsilon) - X(p)}{\varepsilon}\right] = \EE[\delta + w(Y-X(p))] = \EE[\tilde{X}(p)],
\end{equation}
as desired.
\end{proof}

\subsection{Construction of stochastic derivatives for elementary programs}

We now give a technical condition under which a stochastic derivative can be constructed for an elementary stochastic program, in which we characterize $w$ as the derivative of the conditional probability of $A_B(\varepsilon)$ given $X(p)$, and $Y$ as realization of the weak limit of the conditional distribution of $X(p+\varepsilon)$ given $X(p)$ and $A_B(\varepsilon)$. 

\begin{assumption}\label{ass:limitingY}
We assume that $X(p)$ is almost surely differentiable, so that $\dd X(\varepsilon) / \varepsilon \to \delta$ almost surely as $\varepsilon \to 0$ for some~$\delta$, and that we may find an integrable random bound $B > |\delta|$ such that
\begin{itemize}
    \item for the quantity
    \begin{equation}
        w(\varepsilon) = \PP\left(A_B(\varepsilon) \mid X(p)\right),
    \end{equation}
    $w(\varepsilon)/\varepsilon$ is dominated in $\varepsilon$ by an integrable random variable and converges almost surely to a random variable $w$ as $\varepsilon \to 0^{+/-}$.\item the conditional distribution of $X(p+\varepsilon)$ given $X(p)$ and $A_B(\varepsilon)$ converges in distribution to the distribution of a random variable $Y$ as $\varepsilon \to 0^{+/-}$. Specifically, $Y$ must satisfy\begin{equation}\label{eq:limitingY}
    \EE\left[w \cdot f(Y) \xmiddle| X(p) \right] = \lim_{\varepsilon \to 0^{+/-}} \EE\left[w \cdot f(X(p+\varepsilon)) \xmiddle| X(p), A_B(\varepsilon)\right]
\end{equation}
for all bounded continuous functions $f\colon E\to \RR$.
\end{itemize}
\end{assumption}

Given \cref{ass:limitingY}, we may construct a stochastic derivative for $X(p)$. 
\newtheorem*{thmcopy}{\cref{thm:main}}
\begin{thmcopy}
Given a stochastic program $X(p)$ satisfying \cref{ass:limitingY}, there exists a right stochastic derivative $(\delta, w_R, Y_R)$ with $w_R \geq 0$ at any $p \in [a,b)$ and a left stochastic derivative $(\delta, w_L, Y_L)$ with $w_L \le 0$ at any $p \in (a,b]$.
\label{thm:mainfull}
\end{thmcopy}
\begin{proof}
Construct $(\delta, w, Y)$ via \cref{ass:limitingY} (taking the limit from above for the case of right stochastic derivatives, and the limit from below for the case of left stochastic derivatives). For bounded $f$ with bounded derivative, write
\begin{align}
&\EE \left[\frac{f(X(p+\varepsilon)) - f(X(p))}{\varepsilon} \ind_{A_B(\varepsilon)} \xmiddle| X(p) \right] \\ &\quad=
\EE\left[w \left(f(X(p+\varepsilon)) - f(X(p))\right) \frac{\ind_{A_B(\varepsilon)}}{w(\varepsilon)} \xmiddle | X(p) \right] \notag \\
&\qquad+ \EE\left[\left(\frac{w(\varepsilon)}{\varepsilon} - w\right)  \left(f(X(p+\varepsilon)) - f(X(p))\right) \frac{\ind_{A_B(\varepsilon)}}{w(\varepsilon)} \xmiddle| X(p) \right].
\end{align}
Note that $\EE[\ind_{A_B(\varepsilon)} \mid X(p)] = \PP(A_B(\varepsilon) \mid X(p)) = w(\varepsilon)$. Thus, as $\varepsilon \to 0^{+/-}$ the first term approaches
\begin{align}
&\EE\left[w \left(f(X(p+\varepsilon)) - f(X(p))\right) \frac{\ind_{A_B(\varepsilon)}}{w(\varepsilon)} \xmiddle| X(p) \right]
\\
&\qquad=
\EE\left[w  \left(f\left(X(p+\varepsilon)\right) - f(X(p))\right) \xmiddle| X(p), A_B(\varepsilon) \right] \\
&\qquad\to \EE\left[w \left(f(Y) - f(X(p))\right) \xmiddle| X(p)\right]
\end{align}
by \cref{ass:limitingY}. 
Now, using the $X(p)$-measurability of $w(\varepsilon)$ and $w$, we may bound the magnitude of the second term as\begin{align}
&\left|\EE\left[\left(\frac{w(\varepsilon)}{\varepsilon} - w\right)  \left(f(X(p+\varepsilon)) - f(X(p))\right) \frac{\ind_{A_B(\varepsilon)}}{w(\varepsilon)} \xmiddle| X(p)\right]\right|
\\
&\qquad\leq 2\|f\|_\infty \cdot  \left|\frac{w(\varepsilon)}{\varepsilon} - w\right|  \cdot \EE\left[ \frac{\ind_{A_B(\varepsilon)}}{w(\varepsilon)}\xmiddle| X(p)\right] 
\\
&\qquad= 2\|f\|_\infty \cdot  \left|\frac{w(\varepsilon)}{\varepsilon} - w\right|,
\end{align}
which approaches 0 almost surely as $\varepsilon \to 0^{+/-}$ by the almost sure convergence of $w(\varepsilon)/\varepsilon$ to $w$. Thus,
\begin{align}
    \EE\left[w \left(f(Y) - f(X(p)\right) \xmiddle| X(p)\right] = \lim_{\varepsilon \to 0^{+/-}} \EE \left[\frac{f(X(p+\varepsilon)) - f(X(p))}{\varepsilon} \ind_{A_B(\varepsilon)} \xmiddle| X(p) \right].
\end{align}
Additionally,   $\PP\left(A_B(\varepsilon) \xmiddle| X(p)\right)/\varepsilon$ is dominated in $\varepsilon$ by an integrable random variable 
 and $\dd X(\varepsilon)/\varepsilon \to \delta$ almost surely by assumption. So $(\delta, w, Y)$ is a stochastic derivative of $X$ at $p$. Finally, note that $w(\varepsilon)/\varepsilon \geq 0$ for $\varepsilon > 0$ while $w(\varepsilon)/\varepsilon \leq 0$ for $\varepsilon < 0$, by the non-negativity of $w(\varepsilon)$. So applying \cref{ass:limitingY} with limits from above for $p \in [a,b)$ indeed produces a right stochastic derivative, while applying \cref{ass:limitingY} with limits from below for $p \in (a,b]$ produces a left stochastic derivative.
\end{proof}

\subsection{Composition of stochastic derivatives}

We now provide a proof of the chain rule for stochastic derivatives.

\newtheorem*{compcopy}{\cref{prop:composition}}
\begin{compcopy}[Chain rule]
Consider independent stochastic programs $X_1$ and $X_2$ and their composition $X_2 \circ X_1$. Suppose that $X_1$ has a right (left) stochastic derivative at $p \in \mathbb{R}$ given by $(\delta_1, w_1, Y_1)$ with bound $B_1$, and $X_2$ has a right stochastic derivative $(\delta_2, w_2, Y_2)$ in the direction\footnote{when $\delta_1 = 0$, we may choose $\widehat{\delta}_1$ arbitrarily, as $\delta_2, w_2$ and $Y_2$ ultimately have no contribution to the stochastic derivative in this case.} $\widehat{\delta}_1 = \delta_1 / |\delta_1|$ with bound $B_2$ given conditionally on its input $X_1(p)$, where with $\dd X_2(\mathbf{v}) = X_2(X_1(p) + \mathbf{v}) - X_2(X_1(p))$ the event
\begin{equation}
\tilde{A}_2(\varepsilon) = \{|\dd X_2(\mathbf{v})| > B_2|\mathbf{v}|\text{ for some }\mathbf{v}\text{ satisfying }|\mathbf{v}| \leq \varepsilon\}
\end{equation}
has $\PP\left(\tilde{A}_2(\varepsilon) \xmiddle| X_1(p), X_2(p)\right) \big/\varepsilon$ dominated in $\varepsilon$ by an integrable random variable $W$, and $\dd X_2(\mathbf{v})$ is almost surely differentiable with respect to $\mathbf{v}$.
\ Then, if $B_1 B_2$ and $|\delta_1| W$ are integrable, the stacked program $[X_1; X_2 \circ X_1]$ has a right (left) stochastic derivative at $p$ given by $(\delta, w, Y)$ where $\delta = \left[\delta_1 ; |\delta_1|\delta_2\right]$, \begin{equation} Y \quad = \quad \begin{cases}
\tag{\ref{eq:cases}}
    [Y_1; X_2(Y_1)] & \text{ with probability }\quad\dfrac{w_1}{w_1 + |\delta_1| w_2}, \\
    [X_1(p); Y_2] & \text{ with probability }\quad \dfrac{|\delta_1| w_2}{w_1 + |\delta_1| w_2},
    \end{cases}
\end{equation}
and $w = w_1 + |\delta_1| w_2$, with associated bound $B =B_1 +  B_1 B_2$. 
\end{compcopy}
\begin{proof}
    We prove the case of right stochastic derivatives, as the case of left stochastic derivatives is analogous. Let $X = X_2 \circ X_1$ as a shorthand. Now, for $\varepsilon > 0$ consider the events
\begin{align}
    A_1(\varepsilon) &=  \{|\dd X_1(\varepsilon)| > B_1 \varepsilon\},\\
A_2(\varepsilon) &= \{|\dd X_2(\varepsilon \delta_1)| >  |\delta_1|B_2\varepsilon\}, \\ 
    A(\varepsilon) &=  \{|\dd X(\varepsilon)| > B_1 B_2 \varepsilon\}.
\end{align}

Intuitively, the events $A_1(\varepsilon)$, $A_2(\varepsilon)$, and $A(\varepsilon)$ represent large jumps in $X_1$, $X_2$, and $X$ respectively. Throughout the proof, we let $f$ denote a bounded function with bounded derivative. 

\underline{Step 1: The composed program is almost surely differentiable.}

Since $X_1(p)$ and $\dd X_2(\mathbf{v})$ are almost surely differentiable, the chain rule implies that $X(p) = X_2(X_1(p))$ is almost surely differentiable, with derivative $|\delta_1| \delta_2$, i.e. $\dd X(\varepsilon)/\varepsilon \to |\delta_1| \delta_2$ almost surely as $\varepsilon \to 0$. Additionally, the directional perturbation $\varepsilon \mapsto X_2(X_1(p) + \varepsilon \delta_1)$ is also almost surely differentiable by the almost-sure differentiability of $X_1(p)$ and the definition of $\delta_2$, with the same derivative $|\delta_1| \delta_2$.

\underline{Step 2: No large jump in $X_1$ or $X_2$ implies no large jump in $X$.}

Given $A_1^c(\varepsilon)$ and $\tilde{A}_2^c(\varepsilon |\delta_1|)$, we have 
\begin{align}
    |\dd X(\varepsilon)| &= |\dd X_2(\dd X_1(\varepsilon))| \\
    \label{eq:goodbound}
                    &\leq B_1 B_2 \varepsilon,
\end{align}
which implies $A^c(\varepsilon)$. We now remark that $\tilde{A}_2(\varepsilon |\delta_1|)$ and $A_2(\varepsilon)$ may be interchanged when considering the sensitivity of the directional perturbation $\varepsilon \mapsto X_2(X_1(p) + \varepsilon \delta_1)$. To see this, note that 
\begin{equation}
\tilde{A}_2^c(\varepsilon |\delta_1|) \subseteq A_2^c(\varepsilon) = \{\left|X_2(X_1(p) + \varepsilon \delta_1) - X(p)\right| \leq  |\delta_1| B_2 \varepsilon\}.
\end{equation}
Thus, since $X_2(X_1(p) + \varepsilon \delta_1)$ is almost surely differentiable,  we may apply the dominated convergence argument of \cref{prop:repara} to both $A_2(\varepsilon)$ and $\tilde{A}_2(\varepsilon |\delta_1|)$, obtaining
\begin{align}
    &\lim_{\varepsilon \to 0^+} \EE\left[\frac{f(X_2(X_1(p)+\varepsilon\delta_1)) - f(X(p))}{\varepsilon} \ind_{A_2^c(\varepsilon)}\xmiddle| X_1(p), X(p)\right] \\
    &\qquad=\lim_{\varepsilon \to 0^+} \EE\left[f'(X_2(X_1(p))) \delta_2 |\delta_1| \xmiddle| X_1(p), X(p)\right] \\
    &\qquad= \lim_{\varepsilon \to 0^+} \EE\left[\frac{f(X_2(X_1(p)+\varepsilon\delta_1))) - f(X(p))}{\varepsilon}  \ind_{\tilde{A}_2^c(\varepsilon |\delta_1|)}\xmiddle| X_1(p), X(p)\right] \label{eq:exchange2}.
\end{align}

By a similar argument, it holds that $A_1(\varepsilon) \cup \tilde{A}_2(\varepsilon|\delta_1|)$ may be interchanged with $A(\varepsilon)$ when considering the sensitivity of $X_2 \circ X_1$. Specifically, since $X_2(X_1(p))$ is almost surely differentiable and both $A_1^c(\varepsilon) \cap \tilde{A}_2^c(\varepsilon|\delta_1|)$ and $A_2^c(\varepsilon)$ bound $|\dd X(\varepsilon)|$ by $B_1 B_2\varepsilon$, we may apply \cref{prop:repara} for $X = X_2 \circ X_1$ to both event classes, obtaining
\begin{align}
    &\lim_{\varepsilon \to 0^+} \EE\left[\frac{f(X(p+\varepsilon)) - f(X(p))}{\varepsilon} \ind_{A_1^c(\varepsilon) \cap \tilde{A}_2^c(\varepsilon|\delta_1|)}\xmiddle| X_1(p), X(p)\right] \\
    &\qquad=\lim_{\varepsilon \to 0^+} \EE\left[f'(X_2(X_1(p))) \delta_2 |\delta_1| \xmiddle| X_1(p), X(p)\right] \\
    &\qquad= \lim_{\varepsilon \to 0^+} \EE\left[\frac{f(X(p+\varepsilon)) - f(X(p))}{\varepsilon}  \ind_{A_2^c(\varepsilon)}\xmiddle| X_1(p), X(p)\right]. \label{eq:exchange3}
\end{align}

\underline{Step 3: A large jump in both $X_1$ and $X_2$ has negligible probability.} 

Note that $A_1(\varepsilon)$ and $\tilde{A}_2(\epsilon|\delta_1|)$ are independent conditional on $X_1(p)$. Therefore,
in the limit, we have 
\begin{align}
& \frac{\PP\left(A_1(\varepsilon) \cap \tilde{A}_2(\varepsilon|\delta_1|) \xmiddle| X_1(p), X(p)\right)}{\varepsilon} \\
&\qquad= \frac{\PP\left(A_1(\varepsilon) \xmiddle| X_1(p) \right)}{\varepsilon} \PP\left(\tilde{A}_2(\varepsilon|\delta_1|) \xmiddle| X_1(p), X(p) \right) \\
&\qquad\to 0
\end{align}
almost surely as $\varepsilon \to 0^+$ by the domination of $\PP\left(A_1(\varepsilon) \xmiddle| X_1(p)\right)  \big/ \varepsilon$ by an integrable random variable and that $\PP\left(\tilde{A}_2(\varepsilon|\delta_1|) \xmiddle| X_1(p), X(p)\right) \to 0$. By boundedness of $f$ this implies
\begin{equation}
    \label{eq:neglect}
    \lim_{\varepsilon \to 0^+} \EE\left[\frac{f(X(p+\varepsilon)) - f(X(p))}{\varepsilon} \ind_{A_1(\varepsilon) \cap \tilde{A}_2(\varepsilon|\delta_1|)} \xmiddle| X_1(p), X(p) \right] = 0,
\end{equation}
almost surely, so the sensitivity of $X$ is negligible under $A_1(\varepsilon) \cap \tilde{A}_2(\varepsilon|\delta_1|)$.

\underline{Step 4: Sensitivity of $X(p)$ given a large jump in $X_1$ is described by $w_1, Y_1$.}

The random variable
\begin{equation}
    Z(y) = \EE[f(X_2(y))   \mid X_1(p), X(p)]
\end{equation}
is $X_1(p)$, $X(p)$ measurable and almost surely continuous, and the function $f$ in the characterization \cref{eq:deriv} can be taken as bounded and almost surely continuous in $X_1(p)$ by an extension of the Portmanteau theorem \cite{MR1377542}.
It follows, as
$X(p)$ and $X_1(p+\epsilon)$ are independent given $X_1(p)$, that
\begin{align}
& \lim_{\varepsilon \to 0^+} \EE\left[\frac{Z(X_1(p+\varepsilon))-Z(X_1(p))}{\varepsilon} \ind_{A_1(\varepsilon)} \xmiddle| X_1(p), X(p) \right]
 \\   &\qquad= \EE\left[w_1\left( Z(Y_1)-Z(X_1(p))\right)  \xmiddle| X_1(p), X_2(p) \right],
\end{align}
so that
\begin{align}
&\lim_{\varepsilon \to 0^+} \EE\left[\frac{f(X(p + \varepsilon)) - f(X(p))}{\varepsilon}  \ind_{A_1(\varepsilon)}\xmiddle|X_1(p), X(p) \right]
\\&\qquad =
 \EE\left[w_1\left(f(X_2(Y_1)) - f(X(p))\right)\xmiddle|X_1(p), X(p) \right].
\end{align}

\underline{Step 5: Sensitivity of $X(p)$ given a large jump in $X_2$ is described by $w_2, Y_2$.} 

Using \cref{eq:goodbound}, given $\tilde{A}_2^c(\varepsilon|\delta_1|) \cap A_1^c(\varepsilon)$ it holds that
\begin{align}
\left|\frac{f\left(X_2(X_1(p) + \dd X_1(\varepsilon))\right) - f(X_2(p))}{\varepsilon}\right| \leq \|f'\|_\infty \frac{|\dd X(\varepsilon)|}{\varepsilon} \leq \|f'\|_\infty B_1 B_2
\end{align}
is dominated by an integrable random variable. Therefore, since $\dd X_1(\varepsilon)/\varepsilon \to \delta_1$ almost surely,
\begin{align}
    &\lim_{\varepsilon \to 0^+} \EE\left[\frac{f(X_2(X_1(p)+\dd X(\varepsilon))) - f(X(p))}{\varepsilon} \ind_{\tilde{A}_2^c(\varepsilon|\delta_1|) \cap A_1^c(\varepsilon)}\xmiddle| X_1(p), X(p)\right] \\
    \qquad&=
    \lim_{\varepsilon \to 0^+} \EE\left[\frac{f(X_2(X_1(p)+\varepsilon\delta_1)) - f(X(p))}{\varepsilon} \ind_{\tilde{A}_2^c(\varepsilon|\delta_1|) \cap A_1^c(\varepsilon)}\xmiddle| X_1(p), X(p)\right]
\end{align}
by dominated convergence. Putting this together with \cref{eq:exchange2}, by definition of $w_2$ and $Y_2$, 
\begin{align}
    &\EE\left[|\delta_1| w_2\left(f(Y_2) - f\big(X(p)\big)\right) \xmiddle| X_1(p), X(p) \right] \\
    &\quad = \lim_{\varepsilon \to 0^+} \EE\left[\frac{f(X_2(X_1(p)+\varepsilon\delta_1)) - f(X(p))}{\varepsilon} \ind_{A_2(\varepsilon)}\xmiddle| X_1(p), X(p)\right] \\
    &\quad = \lim_{\varepsilon \to 0^+} \EE\left[\frac{f(X_2(X_1(p)+\varepsilon\delta_1))) - f(X(p))}{\varepsilon}  \ind_{\tilde{A}_2(\varepsilon|\delta_1|)}\xmiddle| X_1(p), X(p)\right] \\
    &\quad = \lim_{\varepsilon \to 0^+} \EE\left[\frac{f(X_2(X_1(p)+\varepsilon\delta_1))) - f(X(p))}{\varepsilon}  \ind_{\tilde{A}_2(\varepsilon|\delta_1|) \cap A_1^c(\varepsilon)}\xmiddle| X_1(p), X(p)\right] \label{eq:beforedom} \\
    &\quad = \lim_{\varepsilon \to 0^+} \EE\left[\frac{f(X_2(X_1(p)+\dd X_1(\varepsilon))) - f(X(p))}{\varepsilon}  \ind_{\tilde{A}_2(\varepsilon|\delta_1|) \cap A_1^c(\varepsilon)}\xmiddle| X_1(p), X(p) \right] \label{eq:afterdom} \\
    &\quad = \lim_{\varepsilon \to 0^+} \EE\left[\frac{f(X(p+\varepsilon)) - f(X(p))}{\varepsilon}  \ind_{\tilde{A}_2(\varepsilon|\delta_1|)}\xmiddle| X_1(p), X(p)\right],
\end{align}
where we freely neglect (or add back in)  $\ind_{A_1(\varepsilon) \cap \tilde{A}_2(\varepsilon|\delta_1|)}$ in the conditional expectation by \cref{eq:neglect}.

\underline{Step 6: $(\delta, w, Y)$ is a stochastic derivative of the stacked program $[X_1(p); X(p)]$.}

Since $\delta_1$ is an almost-sure derivative of $X_1(p)$ and $|\delta_1| \delta_2$ is an almost-sure derivative of $X(p)$, $\delta = [\delta_1; |\delta_1|\delta_2]$ is an almost sure derivative of $[X_1(p); X(p)]$.

Given $A(\varepsilon)$, $|\dd X(\varepsilon)| > B_1 B_2 \varepsilon$, and given $A_1(\varepsilon)$, $|\dd X_1(\varepsilon)| > B_1 \varepsilon$. Thus we may choose the bound $B = B_1 B_2 + B_1$ for the stacked program, so that $A_B(\varepsilon) \subseteq A_1(\varepsilon) \cup A(\varepsilon)$.  Now, since $A(\varepsilon) \subseteq A_1(\varepsilon) \cup \tilde{A}_2(\varepsilon|\delta_1|)$,  
\begin{align}
\frac{\PP\left(A(\varepsilon) \xmiddle| X_1(p), X(p)\right)}{\varepsilon} &\leq \frac{\PP\left(A_1(\varepsilon) \xmiddle| X_1(p), X(p)\right)}{\varepsilon} + \frac{\PP\left(\tilde{A}_2(\varepsilon|\delta_1|) \xmiddle| X_1(p), X(p)\right)}{\varepsilon} \\
&\leq \frac{\PP\left(A_1(\varepsilon) \xmiddle| X_1(p), X(p)\right)}{\varepsilon} + |\delta_1| W,
\end{align}
where $\PP\left(A_1(\varepsilon) \xmiddle| X_1(p), X(p)\right) \big/ \varepsilon$ and $|\delta_1| W$ are integrable by assumption. Therefore, $\PP\left(A_B(\varepsilon) \xmiddle| X_1(p), X(p) \right)$ is also dominated by an integrable random variable, as desired.

Now, let $f$ operate on the stacked program, and for convenience write $f(x_1; x_2) = f([x_1; x_2])$. As a shorthand, let $\overline{X}(p) = [X_1(p); X(p)]$. With $w$ and $Y$ as given in the statement,\begin{align}
    &\EE\left[w\big(f(Y) - f(X_1(p); X(p))\big) \xmiddle| X_1(p), X(p) \right] \\
    &\quad=  \EE\left[w_1\big(f\left(X_1(p); X_2(Y_1)\right) - f(X_1(p); X(p))\big) \xmiddle| X_1(p), X(p) \right] \notag \\
    &\qquad +  \EE\left[|\delta_1| w_2\big(f(Y_1; X_2(Y_1)) - f(X_1(p); X(p))\big) \xmiddle| X_1(p), X(p) \right] \\
    &\quad = \lim_{\varepsilon \to 0^+} \EE\left[\frac{f(\overline{X}(p+\varepsilon)) - f(\overline{X}(p))}{\varepsilon}  \left(\ind_{A_1(\varepsilon)} + \ind_{\tilde{A}_2(\varepsilon|\delta_1|)} \right)\xmiddle|X_1(p), X(p) \right] \\
    &\quad = \lim_{\varepsilon \to 0^+} \EE\left[\frac{f(\overline{X}(p+\varepsilon)) - f(\overline{X}(p))}{\varepsilon}  \ind_{A_1(\varepsilon) \cup \tilde{A}_2(\varepsilon|\delta_1|)} \xmiddle|X_1(p), X(p) \right] \\
    &\quad = \lim_{\varepsilon \to 0^+} \EE\left[\frac{f(\overline{X}(p+\varepsilon)) - f(\overline{X}(p))}{\varepsilon}  \ind_{A_1(\varepsilon) \cup A(\varepsilon)}\xmiddle|X_1(p), X(p) \right] \\
    &\quad = \lim_{\varepsilon \to 0^+} \EE\left[\frac{f(\overline{X}(p+\varepsilon)) - f(\overline{X}(p))}{\varepsilon}  \ind_{A_B(\varepsilon)}\xmiddle|X_1(p), X(p) \right].
\end{align}
where we use \cref{eq:exchange3} in the penultimate equality, and the last equality follows from applying \cref{prop:repara} to both $A_B(\varepsilon)$ and $A_1(\varepsilon) \cup A(\varepsilon)$ as in \cref{eq:exchange3}, noting that $A_B^c(\varepsilon)$ bounds $|\dd \overline{X}(\varepsilon)|$ by $B\varepsilon$ and that $A_1^c(\varepsilon) \cap A_2^c(\varepsilon) \subseteq A_B^c(\varepsilon)$. We conclude that $(\delta, w, Y)$ is a valid stochastic derivative of $\overline{X}(p) = [X_1(p); X(p)]$. 

\end{proof}

\subsection{Unbiasedness of pruning strategy}
\label{sec:pruning_proof}
In the main text, we note that we can employ a pruning strategy so that we only ever track one alternative path (i.e. one sample from the stochastic derivative component $Y$) and yet still obtain an unbiased estimate. The following shows that the pruning method, whereby one chooses between two samples of $Y$ by picking one with probability proportional to its weight, is indeed unbiased.

We prove by induction that the currently tracked alternative path is an unbiased choice amongst all possible alternative paths seen so far. The base case, where the first alternative path observed is followed, is trivial as it is the only choice. Now, suppose that $n$ alternative branches have been observed so far, where $w = w_1 + w_2 + \dots + w_n$ is the summed weight so far. Suppose we observe an $(n+1)$th branch, with weight $w_{n+1}$. By our pruning strategy it is chosen with probability
\begin{equation}
    \frac{w_{n+1} }{w_{n+1} + w} = \frac{w_{n+1}}{w_1 + w_2 + \dots + w_{n+1}},
\end{equation}
while the $j$th path for some $j \leq n$ will have been chosen after this step with probability
\begin{equation}
    \frac{w_j}{w_1 + w_2 + \dots + w_n} \cdot \frac{w_1 + w_2 + \dots + w_n}{w_1 + w_2 + \dots + w_{n+1}} = \frac{w_j}{w_1 + w_2 + \dots + w_{n+1}},
\end{equation}
as desired for unbiasedness. Note that this enables us to prune \emph{online}, i.e.~without knowing the full structure of the computation a priori, which we exploit in \texttt{StochasticAD.jl} for $\mathcal{O}(1)$ memory overhead.

\subsection{Smoothing}
\label{app:smoothing}
Recall from the main text,
\newtheorem*{defsmoothcopy}{\cref{def:smooth}}
\begin{defsmoothcopy}[Smoothed stochastic derivative]
For a stochastic program $X(p)$ with a right (left) stochastic derivative $(\delta, w, Y)$ at input $p$, a right (left) smoothed stochastic derivative $\tilde{\delta}$ of $X$ at input $p$ is given as
\begin{equation}
\tag{\ref{eq:defsmooth}}
    \tilde \delta = \EE\left[\delta + w (Y- X(p)) \xmiddle| X(p)\right].\end{equation}
\end{defsmoothcopy}
Using \cref{def:deriv}, we can write the above in an alternative form that does not rely on the definition of the stochastic derivative,
\begin{align}
    \tilde \delta  &= \lim_{\varepsilon \to 0^{+/-}} \frac{\EE\left[\dd X(\varepsilon) \xmiddle| X(p)\right]}{\varepsilon}, \end{align}
as given in~\cite{gong1987smoothed,glasserman1990smoothed}.
\ Smoothed stochastic derivatives propagate through functions that are \emph{locally} linear over the range of $Y$ conditionally on $X(p)$, as we now formalize. Note that this is a much weaker requirement than global linearity over the full range of $X(p)$, as the range of $Y$ conditional on $X(p)$ is generally more restricted (e.g. $Y \in \{X(p)-1, X(p)+1\}$ for a binomial variable).

\begin{proposition}
Suppose that $f$ is linear over the range of\/ $Y$ conditionally on $X(p) = x$, for all $x$. Then,
\label{prop:lastone}
\begin{equation}
\pder{p} \EE f(X(p)) = \EE\left[ f'(X(p))\tilde\delta \right].
\end{equation}
\end{proposition}
\begin{proof}
Note that $(f'(X(p)) \cdot \delta, w, f(Y))$ is a stochastic derivative of $f \circ X$ by \cref{prop:composition}. Now, by \cref{prop:unbiased} and the local linearity of $f$, we have the simplification, 
\begin{align}
\pder{p} \EE \left[f(X(p))\right] &= \EE\left[ f'(X(p)) \cdot \delta + w \left(f(Y) - f(X(p))\right)\right] \\
&= \EE\left[f'(X(p)) \left(\delta + w (Y - X(p))\right)\right] \\
&= \EE\left[ f'(X(p))\tilde\delta \right].
\end{align}
\end{proof}

In most cases, local linearity will only hold approximately, leading to bias in the estimate produced by propagating smoothed stochastic derivatives. However, in the case of a particle filter resampling step an exact estimate is produced, as we show in the following.

%% file: filter.tex
\section{Formalism of the particle filter}
\label{sec:SI_filter}

\subsection{Hidden Markov model}

Let us consider a hidden Markov model with random states $X_1, \dots, X_n$ as specified by a stochastic program $X_1(\theta)$ giving the starting value depending on parameters $\theta$ and consecutive states given by pointwise differentiable stochastic programs $X_i(x_{i-1}, \theta)$ depending on the previous state $x_{i-1} \sim X_{i-1}$ (Markov property) and $\theta$.
In general, we allow continuous probability densities $p(x_1\mid \theta)$ of $X_1$ and continuous transition probability densities $p(x_i\mid x_{i-1}, \theta)$ to depend arbitrarily on $\theta$. 
This latent process $X_1, \dots, X_n$ is indirectly observed as the process $y_1 \sim Y_1, \dots, y_n \sim Y_n$ with $n$ observations $Y_i(x_i)$ depending on $x_i \sim X_i$ which are specified to have smooth conditional probability densities $p(y_i \mid x_i,\theta)$ depending only on $x_i$ and $\theta$. 
As a concrete, special case, we consider the following linear Gaussian state-space model with a $d$-dimensional latent process, 
\begin{align}
X_i &= \Phi X_{i-1} + \operatorname{Normal}(0,Q),\label{eq:hidden_markov_model_1}\\
Y_i &= X_i + \operatorname{Normal}(0,R),
\label{eq:hidden_markov_model_2}
\end{align}
where $Q = 0.02 \cdot \id_{d \times d}$, $R=0.01 \cdot \id_{d \times d}$, $x_1 \sim \operatorname{Normal}(\mu,0.001 \cdot \id_{d \times d})$, $\mu \sim \operatorname{Normal}(0,\id_{d \times d})$ is a random initial position,  and $\Phi$ is a $d$-dimensional rotation matrix. Here, the parameters $\theta$ are defined by the entries of $\Phi$, i.e. $\theta = \operatorname{vec}(\Phi)$.
We use a particle filter to compute an estimate of the likelihood $\mathcal{L}= p(y_1, \dots, y_n \mid \theta)$.

\subsection{Differentiating a particle filter with a  resampling step}

Given $n$ observations $y_1 \sim Y_1 , \dots, y_n \sim Y_n$ of the hidden Markov model defined via Eqs.~\eqref{eq:hidden_markov_model_1} and \eqref{eq:hidden_markov_model_2}, a bootstrap particle filter allows us to approximate the posterior distributions of the states and the likelihood $\mathcal{L}$ of these observations by propagating a cloud of $K$ \emph{weighted} particles. We denote the $k$th particle by $x_i^{(k)}$.
Each particle evolves independently according to the stochastic program [Eqs.~\eqref{eq:hidden_markov_model_1} and \eqref{eq:hidden_markov_model_2}] and carries a weight $\varpi_i^{(k)}$ measuring how well the trajectory of the particle so far is matching the observations. This importance weight is updated by $\varpi^{(k)}_i = p(y_i \mid x^{(k)}_i, \theta) \varpi^{(k)}_{i-1}$, such that the empirical measure $\sum_k \varpi^{(k)}_i \delta_{x^{(k)}_i}/\sum_k \varpi^{(k)}_i$
of the weighted particles approximates the filtering distribution $p(x_i\mid y_1, \dots, y_i, \theta)$. 

So far, the particles and weight trajectories are differentiable with respect to the parameter~\citeSI{jonschkowski2018differentiable_2}.
However, weight degeneracy, the collapse of all but a few weights, is a common problem. Our goal is to discard unlikely particles, so that numerical resources are not wasted on particles with vanishing weight. The strategy to accomplish this goal is to include resampling steps, where we pick the particles that best match the observations. However, such resampling steps present discrete randomness, where the particle population is resampled according to the particle weights $\varpi_i^{(k)}$ to form a new population $(x')_i^{(k)}$ with equal weight $\varpi'_i = 1/K \cdot \sum_{k=1}^K \varpi_i^{(k)}$. 
To differentiate the resampling step, we need to describe how perturbing the weight distribution $\varpi_i^{(k)}$ provided to the resampling procedure affects the resampled particles and weights. 
Importantly, the marginal likelihood of a parameter $\theta$ can be read off using the weights at the last step
\begin{equation}
\label{eq:likelihood_pf}
    p(y_1, \dots, y_n \mid \theta) \approx \sum_{k=1}^K \varpi_n^{(k)}.
\end{equation} 

The strategies for resampling vary, but what they have in common is that the \emph{marginal} distribution of each resampled particle $x^{\prime(k)}_i$ is a multinomial distribution over the original particles with weights given by the normalized weight vector $ (\varpi^{(1)}_i/\varpi_i, \dots, \varpi^{(K)}_i/\varpi_i)$ with $\varpi_i = \sum_{k=1}^K \varpi^{(k)}_i$. Thus, a weighted sample of the marginal distribution of each resampled particle can be obtained by repeating the following procedure until obtaining a weight of $\varpi_i$: 
\begin{enumerate}
    \item Sample an integer $k$ uniformly from $1$ to $K$.
    \item Return the particle $x_i^{(k)}$ with assigned weight $\varpi_i \operatorname{Ber}(\varpi_i^{(k)}/\varpi_i)$.
\end{enumerate}
Conditional on the assigned weight being $\varpi_i$, the returned particle obeys the multinomial distribution. 
The key insight is that the alternate possibility of a resampled particle $x_i^{(k)}$ not being chosen can be written instead as the alternate possibility of its weight $\varpi_i^{(k)}$ changing to 0.

Recall from \cref{ex:ber} that there is an asymmetry between the left and right stochastic derivative of the Bernoulli distribution $\operatorname{Ber}(\varpi_i^{(k)}/\varpi)$ conditioned on the output 0 or 1 being chosen. 
In this case, it is convenient to take the \emph{left} stochastic derivative $(0, w_i^{(k)}, Y_i^{(k)})$ with $Y_i^{(k)} = 0$ and $w_i^{(k)} = -\varpi_i^k / \varpi_i$ of the particle's weight $\varpi_i$ because it is 0 when the assigned weight is 0, meaning that this case has no influence on the derivative estimate. Therefore, particles that are not resampled [and thus do not contribute to the primal value, cf. Eq.~\eqref{eq:likelihood_pf}] do (also) not contribute to the derivative computation.
The fact that only the weights have a stochastic derivative, but not the particles imposes the following setting:

Let $X(p)$ be a stochastic program approximated by the program $\tilde X(p)$. Assume we can sample from $\tilde X(p)$. Assuming absolute continuity, we may write expectations as $\EE \varpi(p) f(\tilde X(p))  = \EE f(X(p))$ where $\varpi(p)$ is the Radon-Nikodym derivative~\cite{asmussen2007stochastic} of the law of $\tilde X(p)$ with respect to the law of $X(p)$, evaluated in $\tilde X(p)$.
We thus consider the program which returns the pair of weight and value
\begin{equation}
    (\varpi(p), \tilde X(p)).
\end{equation}

\begin{proposition}\label{prop:feynmankac}
If $\tilde X(p)$ is a continuous program, reparameterized such that it is differentiable pointwise, and $\varpi(p)$ has a smoothed stochastic derivative $\tilde{\delta}$,
\begin{equation}
    \pder{p} \EE f(X(p)) = \EE\left[ \varpi(p)f'(\tilde X(p))\tilde X'(p)
+
 \tilde{\delta} f(\tilde X(p))
\right].
\end{equation}
\end{proposition}
\begin{proof}
By assumption,
\begin{equation}
\EE f(X(p)) = \EE \varpi(p) f(\tilde X(p)).
\end{equation}
As the function $(\varpi, x) \mapsto \varpi f(x)$ is linear in $\varpi$, the statement follows from \cref{prop:lastone}.
\end{proof}

Therefore, we can replace the stochastic derivative of $\varpi \operatorname{Ber}(\varpi_i^{(k)}/\varpi)$ with its smoothed stochastic derivative $w_i^{(k)}(1 - Y_i^{(k)}) = \varpi_i^k / \varpi_i$ obtained using \cref{derived-table}. 
This is convenient, as smoothed stochastic derivatives permit forward- and reverse-mode whereas reverse-mode AD becomes superior than forward-mode AD for functions $f$ from $\mathbb{R}^n$ to $\mathbb{R}^m$ with $m \gg n$~\cite{baydin2018automatic}. 
Since the weights are used in a purely linear fashion, \cref{prop:lastone} and \cref{prop:feynmankac} guarantee that the derivative estimator is unbiased. 

Numerically, we accomplish this in our code with a formally differentiable weight function \texttt{new\_weight(p)} whose primal value is always 1 but whose derivative is the left smoothed stochastic derivative of $\operatorname{Ber}(p)$ for primal output 1,
\begin{equation}
    \tilde{\delta_L} = \frac{1}{p},
\end{equation}
so that a particle $x_i^{(k)}$ has weight given as
\begin{equation}
    \varpi_i \cdot\texttt{new\_weight}(\varpi_i^{(k)}/\varpi_i).
    \label{eq:particleweight}
\end{equation}
In \cite{scibior2021differentiable}, an equivalent expression is derived by different means, using the stop-gradient operator $\bot$: 
\begin{equation}
    \varpi_i \cdot \frac{(\varpi_i^{(k)}/\varpi_i)}  {\bot(\varpi_i^{(k)}/\varpi_i)},
    \label{eq:stopgrad}
\end{equation}
where $\bot$ is formally assigned a derivative of 0.

%% file: backmatter.tex
\section{Implementation details}
\label{sec:implementation}

\subsection{Experiment details}

All computation times in \cref{fig:ps} were measured on an Intel Xeon Platinum 8260 CPU and Julia version 1.6. Garbage collection times are included in the total run time. We provide code and instructions to run the examples in the tutorials folder of \texttt{StochasticAD.jl}.

\subsection{Software dependencies}

\texttt{StochasticAD.jl} is implemented in the Julia Language \citeSI{Julia-2017_2} and uses internally \texttt{Distributions.jl} and \texttt{DistributionsAD.jl}~\citeSI{Distributions.jl-2019,JSSv098i16},
\texttt{ChainRulesCore.jl}~\citeSI{frames_catherine_white_2022_6574605}, and
\texttt{ForwardDiff.jl}~\citeSI{RevelsLubinPapamarkou2016_2}.

The presented examples are using the additional packages:
\texttt{BenchmarkTools.jl}~\citeSI{BenchmarkTools.jl-2016}, 
\texttt{Zygote.jl}~\citeSI{Zygote.jl-2018},
\texttt{GaussianDistributions.jl}~\citeSI{GaussianDistributions}, and \texttt{Plots.jl}~\citeSI{christ2022plots}.

We refer the reader to the documentation of \texttt{StochasticAD.jl} for more detailed information on the package implementation.

\section{Funding details}

This material is based upon work supported by the National Science Foundation under grant no. OAC-1835443, grant no. SII-2029670, grant no. ECCS-2029670, grant no. OAC-2103804, and grant no. PHY-2021825. The information, data, or work presented herein was funded in part by the Advanced Research Projects Agency-Energy (ARPA-E), U.S. Department of Energy, under Award Number DE-AR0001211 and DE-AR0001222. This material is based upon work supported by the Defense Advanced Research Projects Agency (DARPA) under Agreement No HR00112290091. We also gratefully acknowledge the U.S. Agency for International Development through Penn State for grant no. S002283-USAID. The views and opinions of authors expressed herein do not necessarily state or reflect those of the United States Government or any agency thereof. This material was supported by The Research Council of Norway and Equinor ASA through Research Council project "308817 - Digital wells for optimal production and drainage". Research was sponsored by the United States Air Force Research Laboratory and the United States Air Force Artificial Intelligence Accelerator and was accomplished under Cooperative Agreement Number FA8750-19-2-1000. The views and conclusions contained in this document are those of the authors and should not be interpreted as representing the official policies, either expressed or implied, of the United States Air Force or the U.S. Government. The U.S. Government is authorized to reproduce and distribute reprints for Government purposes notwithstanding any copyright notation herein. We also acknowledge financial support
from the NCCR QSIT funded by the Swiss National
Science Foundation (Grant No. 51NF40-185902), as well
as from the Swiss National Science Foundation individual grant (Grant No. 200020\_200481).